\documentclass{article}

\newcommand*\justify{%
  \fontdimen2\font=0.4em
  \fontdimen3\font=0.2em
  \fontdimen4\font=0.1em
  \fontdimen7\font=0.1em
  \hyphenchar\font=`\-
}

\renewcommand{\texttt}[1]{%
  \begingroup
  \ttfamily
  \begingroup\lccode`~=`/\lowercase{\endgroup\def~}{/\discretionary{}{}{}}%
  \begingroup\lccode`~=`[\lowercase{\endgroup\def~}{[\discretionary{}{}{}}%
  \begingroup\lccode`~=`.\lowercase{\endgroup\def~}{.\discretionary{}{}{}}%
  \catcode`/=\active\catcode`[=\active\catcode`.=\active
  \justify\scantokens{#1\noexpand}%
  \endgroup
}
\usepackage[table]{xcolor}
\usepackage{subcaption}
\usepackage{microtype}
\usepackage{graphicx}
\usepackage{booktabs}
\usepackage{url}
\usepackage{nicefrac}
\usepackage{algorithm2e}
\usepackage{comment}
\usepackage{enumitem}
\usepackage{tikz}
\usepackage{tabularx}
\usepackage{multirow}
\usepackage{subcaption}
\usepackage{placeins}
\usepackage{hyperref}
\usepackage{comment}
\usepackage[makeroom]{cancel}
\usepackage{dsfont}
\usepackage{adjustbox}
\usepackage{array}
\usepackage{wrapfig}

\usepackage[preprint]{neurips_2025}

\usepackage{amsmath}
\usepackage{amssymb}
\usepackage{mathtools}
\usepackage{amsthm}
\usepackage{amsfonts}
\usepackage{bbm}
\usepackage{bm}
\usepackage{thmtools}
\usepackage{thm-restate}
\usepackage{comment}

\usepackage[capitalize,noabbrev]{cleveref}

\theoremstyle{plain}
\newtheorem{theorem}{Theorem}[section]
\newtheorem{proposition}[theorem]{Proposition}
\newtheorem{lemma}[theorem]{Lemma}

\theoremstyle{definition}
\newtheorem{definition}[theorem]{Definition}
\newtheorem{assumption}[theorem]{Assumption}
\theoremstyle{remark}
\newtheorem{remark}[theorem]{Remark}

\DeclareMathOperator*{\argmax}{arg\,max}

\newcommand{\algname}{\texttt{Robust-CPD-UCB}\@\xspace}
\newcommand{\algnameshort}{\texttt{R-CPD-UCB}\@\xspace}
\usepackage[textsize=tiny]{todonotes}

\title{Catoni-Style Change Point Detection for Regret Minimization in Non-Stationary Heavy-Tailed Bandits}

\author{%
  Gianmarco Genalti \\
  Politecnico di Milano\\
  \And
  Sujay Bhatt \\ 
  JPMorgan AI Research\\
  \AND 
  Nicola Gatti \\
  Politecnico di Milano \\
  \And 
  Alberto Maria Metelli \\
  Politecnico di Milano
}

\begin{document}

\setlength{\abovedisplayskip}{4pt}
\setlength{\belowdisplayskip}{4pt}
\setlength{\textfloatsep}{8pt}
\setlength{\floatsep}{8pt}

\allowdisplaybreaks

\maketitle

\begin{abstract}
    Regret minimization in stochastic non-stationary bandits gained popularity over the last decade, as it can model a broad class of real-world problems, from advertising to recommendation systems. Existing literature relies on various assumptions about the reward-generating process, such as Bernoulli or subgaussian rewards. However, in settings such as finance and telecommunications, \textit{heavy-tailed} distributions naturally arise. In this work, we tackle the heavy-tailed piecewise-stationary bandit problem. Heavy-tailed bandits, introduced by \citealp{bubeck2013bandits}, operate on the minimal assumption that the finite absolute centered moments of maximum order $1+\epsilon$ are uniformly bounded by a constant $v<+\infty$, for some $\epsilon \in (0,1]$. We focus on the most popular non-stationary bandit setting, \emph{i.e.}, the piecewise-stationary setting, in which the mean of reward-generating distributions may change at unknown time steps. We provide a novel Catoni-style change-point detection strategy tailored for heavy-tailed distributions that relies on recent advancements in the theory of sequential estimation, which is of independent interest. We introduce \algname, which combines this change-point detection strategy with optimistic algorithms for bandits, providing its regret upper bound and an impossibility result on the minimum attainable regret for any policy. Finally, we validate our approach through numerical experiments on synthetic and real-world datasets.
\end{abstract}
\section{Introduction}
In a Multi-Armed Bandit (MAB, for short, \citealp{lattimore2020bandit}), a decision-maker (also called \textit{learner}) is faced with a sequence of repeated decisions among a fixed number of options (or \textit{actions} or \emph{arms}), observing a reward drawn from a probability distribution after each decision. MABs gained popularity over the last two decades, as they allow for strong theoretical guarantees over algorithms' performance while maintaining the model's generality. However, the most traditional MAB model relies on demanding assumptions that are rarely met in the real world.
The majority of the research effort in the MAB literature focused on progressively overcoming these limitations to cover richer scenarios. Examples include \textit{non-stationary} \citep{besbes2014stochastic}, \textit{linear} \citep{abbasi2011improved}, and \textit{heavy-tailed} \citep{bubeck2013bandits} MABs.

In this work, we focus on a broad class of problems that relaxes, at the same time, two core assumptions of the standard MAB problem. In particular, we focus on \textit{heavy-tailed non-stationary} MABs. Our framework allows for a general class of reward-generating probability distributions without relying on parametric assumptions and with a possibly infinite variance, called heavy-tailed distributions. This setting gained popularity over the last decade due to its applications in finance and telecommunications. Moreover, it extends the assumption of sub-gaussian reward distributions, which is customary in the MAB literature. In such application domains, the assumption that reward-generating distributions are fixed along the whole time horizon is too limiting. It is natural to consider settings, such as finance, characterized by non-stationary processes. We address, with a single algorithm named \algname, the problem of learning in non-stationary environments where the noise of the observations can be heavy-tailed. We prove theoretical guarantees over the performance of \algnameshort and show that they are nearly optimal under some mild assumptions. To the best of the authors' knowledge, this is the first work to address the problem of regret minimization in non-stationary bandits under infinite-variance reward distributions. In particular, we face the technical challenge of developing the first change-point detection strategy with proven theoretical guarantees for such types of distributions.
The contributions are organized as follows:
\begin{itemize}[leftmargin=*,noitemsep, topsep=-2pt]
    \item In Section \ref{sec:problem_formulation}, we introduce the definition of the heavy-tailed piecewise-stationary MAB setting. We define the learning problem and introduce a lower bound on the expected regret for this setting.
    \item In Section \ref{sec:technical_preliminaries}, we recall some notions and results from the existing literature on mean estimation for heavy-tailed random variables and on change-point detection.
    \item In Section \ref{sec:algo}, we introduce \algname, an algorithm from regret minimization in our setting. We provide theoretical guarantees on its expected regret and insights on choosing its parameters.
    \item Finally, in Section \ref{sec:experiments} (and Appendix \ref{apx:exp}), we provide numerical evaluations of the performance of \algnameshort, comparing it with baselines from the literature on both real-world and synthetic data.
\end{itemize}
\section{Problem Formulation}
\label{sec:problem_formulation}
In this section, we recall the definitions of heavy-tailed and piecewise-stationary bandit. Then, we introduce the heavy-tailed piecewise-stationary bandits, the focus of this work. We formally define the problem of regret minimization and provide a novel regret lower bound for the problem.

\subsection{Bandit Settings}
Every round $t \in [T]\coloneqq \{1,\ldots,T\}$, a decision $I_t \in [K]$ is undertaken (possibly at random) and a reward $X_{I_t,t}$ is sampled from a probability distribution $\nu_{I_t}$. We call the set $\boldsymbol{\nu} = \{\nu_i\}_{i \in [K]}$ of reward-generating distributions an \textit{instance} of the MAB (from now on, just MAB). Most of the literature deals with reward-generating distributions either \textit{sub-gaussian} or with bounded supports. 

\textbf{Heavy-Tailed Bandits.}~~In \textit{heavy-tailed} bandits \cite{bubeck2013bandits}, the probability distributions $\{\nu_i\}_{i=1}^K$ are \textit{heavy-tailed}.
In this work, we use the same definition of heavy-tailed MAB (HT MAB, for short).
\begin{definition}[Heavy-Tailed MAB]
    Let $X\sim \nu$ be a random variable with support on $\mathbb{R}$. Then, we call $X$ a \textit{heavy-tailed random variable} if it satisfies 
\begin{equation}
\label{eq:heavytail}
    \mathbb{E}_{\nu}[|X-\mathbb{E}_{\nu}[X]|^{1+\epsilon}]\le v,
\end{equation}
    for $\epsilon \in (0,1]$ and $v \in \mathbb{R}^+$. 
    Let $\boldsymbol{\nu}$ be a MAB. Then, if $\boldsymbol{\nu} \in \mathcal{H}_{(v, \epsilon)}^K$, where $\mathcal{H}_{(v, \epsilon)}$ is the set of probability distributions satisfying Equation \eqref{eq:heavytail}, we call $\boldsymbol{\nu}$ a \textit{heavy-tailed bandit} (HT MAB, for short).
\end{definition}
Note that Equation \eqref{eq:heavytail} implies that the variance of the rewards-generating distributions may be infinite (when $\epsilon < 1$). Most of the technical tools employed for sub-gaussian rewards are ineffective for HT MABs.
We address readers to \cite{genalti2024varepsilon} for a recent literature review on HT MABs.

\textbf{Piecewise-Stationary Bandits.}~~In standard MABs, the reward-generating distributions are assumed to never change during learning.
In \textit{non-stationary bandits}, instead,
the reward-generating distributions are dynamic in time, \emph{i.e.}, the rewards of the same arm are
sampled from different distributions depending on the pull time $t \in [T]$. However, if there is no constraint on how many times the distributions may change, then the problem may quickly become non-tractable. Thus, in this work, we consider the most popular non-stationary MAB setting, \emph{i.e.}, the \textit{piecewise-stationary} bandit (PS MAB) from \cite{yu2009piecewise}, where the distributions of rewards remain constant for a certain period, called \textit{epoch}, and then abruptly change at some unknown time points, called \textit{breakpoints}. We assume that the total number of breakpoints $\Upsilon \in [T]$ is fixed before the trial.
We define a PS MAB as follows.
\begin{definition}[Piecewise-Stationary Bandit]
    Let $\{\boldsymbol{\nu}^{(j)}\}_{j \in [\Upsilon]}$ be a set of $\Upsilon$ MABs. Then,  let $\Upsilon$ be a set of timesteps $\{t_c^{(j)}\}_{j\in[\Upsilon]} \subset [T]$ and call $E_j$ the set of indices $\{t_c^{(j-1)},\ldots, t_c^{(j)}\}$, where $t_c^{(0)}=0$ and $t_c^{(\Upsilon+1)}=T$, by convention. If $\nu_i^{(j)}$ is the reward-generating distribution of arm $i \in [K]$ when $t \in E_j$, then  $(\{\boldsymbol{\nu}^{(j)}\}_{j \in [\Upsilon]}, \{t_c^{(j)}\}_{j\in[\Upsilon]})$ defines a \textit{piecewise-stationary bandit} (PS MAB, for short). 
\end{definition}
$E_j$ is the $j$-th epoch of the PS MAB, and $t_c^{(j)}$ to the $j$-th breakpoint. To simplify notation, we define $\mu_i^{(j)}\coloneqq \mathbb{E}_{\nu_i^{(j)}}[X_{i,t}]$ as the mean of the reward-generating distribution of action $i$ during epoch $j$. Note that the reward-generating distribution of an action is fixed during an epoch, and so is the mean reward (and every other distribution parameter). We call $\delta_i^{(j)} \coloneqq |\mu_i^{(j)}-\mu_i^{(j-1)}|$ the magnitude of the change in the mean of arm $i \in [K]$ from epoch $E_{j-1}$ to the next one, $E_j$. By convention, $E_0 = \emptyset$ and $\delta_{i}^{(0)}=\infty$ for every $i \in [K]$. In Appendix \ref{apx:related}, we review the literature on PS MABs.

\textbf{Piecewise Non-stationary Heavy-Tailed Bandits.}~~The general definition of piecewise non-stationary MABs allows for any family of reward-generating distributions, including heavy-tailed ones. In this work, we deal with piecewise non-stationary bandits where the reward-generating distributions satisfy Equation \eqref{eq:heavytail}. We call this setting the \textit{heavy-tailed piecewise-stationary} setting (HTPS, for short). 
\begin{definition}[Heavy-Tailed Piecewise-Stationary Bandits]
\label{defi:htps_mabs}
Let $(\{\boldsymbol{\nu}^{(j)}\}_{j \in [\Upsilon]}, \{t_c^{(j)}\}_{j\in[\Upsilon]})$ be a PS bandit. If $\boldsymbol{\nu}^{(j)} \in \mathcal{H}_{(v,\epsilon)}^K$ for every $j \in [\Upsilon]$, we call it a heavy-tailed piecewise-stationary bandit (HTPS MAB, for short). We denote the set of such HTPS MABs as $\mathcal{B}_{(v,\epsilon,\Upsilon, \boldsymbol{t}_c)}$, where $\boldsymbol{t}_c= \{t_c^{(j)}\}_{j\in[\Upsilon]}$.
\end{definition}
Definition \ref{defi:htps_mabs} introduces the novel bandit setting, as the intersection between HT and PS MABs. To the best of the authors' knowledge, this setting has not been studied in previous literature.
\subsection{Learning Goal} 
A policy is a (possibly randomized) map $\pi(t): \mathcal{F}_{t-1} \mapsto I_t$ that receives the filtration up to time $t-1$ (composed of past actions and rewards) and returns an action $I_t \in [K]$ to play.
We define $\Delta_i^{(j)}$ as the \textit{sub-optimality gap} of arm $i\in [K]$ during epoch $j \in [\Upsilon]$, \emph{i.e.}, $\Delta_i^{(j)}\coloneqq|\max_{k\in[K]}\mu_k^{(j)}-\mu_i^{(j)}|$ and $N_{i,j}^{\pi}(t)$ as the number of times action $i$ has been chosen during epoch $j$ by policy $\pi$ up to time $t \in [T]$. The goal of a learner is to minimize the \textit{expected cumulative regret} $ \mathbb{E}[R^\pi(T)]$, \emph{i.e.}, the cumulative performance gap w.r.t. to the best \textit{policy} over a learning horizon. 
\begin{definition}[Expected Cumulative Regret]
    Given a policy $\pi$, we define the expected cumulative  regret of $\pi$  as:
    \begin{equation*}
        \mathbb{E}[R^\pi(T)] = \sum_{j \in [\Upsilon]}\sum_{i \in [K]} \Delta_i^{(j)}\mathbb{E}[N_{i,j}^{\pi}(T)],
    \end{equation*}
    where the expectation accounts for both the randomness of policy $\pi$ and reward generation. 
\end{definition}
This performance index is also called \textit{dynamic regret} and is the standard choice for PS MABs. The optimal policy corresponds to choosing the best action in every epoch $j \in [\Upsilon]$, \emph{i.e.}, $i^*_j \in \argmax_{i \in [K]} \mu_{k}^{(j)}$.
We also define some quantities that govern the statistical complexity of the instance. $\delta_{min} \coloneqq \min_{i\in [K], j \in [\Upsilon]} \delta_i^{(j)}$ is the minimum change between any two breakpoints, $\Delta_{min}^{(j)} \coloneqq \min_{i\in [K], \Delta_i^{(j)}>0} \Delta_i^{(j)}$ and $\Delta_{max}^{(j)} \coloneqq \max_{i\in [K]} \Delta_i^{(j)}$ are  the minimum and maximum sub-optimality gap during an epoch $j \in [\Upsilon]$, respectively.  Intuitively, the smaller $\delta_{min}^{(j)}$ is, the more difficult it is to detect breakpoints, and the smaller $\Delta_{min}^{(j)}$ is, the more difficult it is to distinguish the best action. On the other hand, the larger these quantities are, the larger the regret potentially incurred with an error. When $(j)$ is omitted, we refer to the quantity minimized/maximized over all epochs.

\subsection{Lower Bound}
In this section, we provide a lower bound to the expected cumulative regret that any policy $\pi$ must incur in an HTPS bandit.
\begin{restatable}[Regret Lower Bound for the HTPS Bandit Problem]{theorem}{regretLowerBound}
  \label{thr:lb}  
  For any fixed policy $\pi$, we have
  \begin{align}
      \label{eq:regret_lb}
      \sup_{\boldsymbol{\nu} \in \mathcal{B}_{(v,\epsilon, \Upsilon)}}\mathbb{E}_{\boldsymbol{\nu}}[R^{\pi}(T)] &\ge \frac{1}{25}(K\Upsilon)^\frac{\epsilon}{1+\epsilon}(vT)^\frac{1}{1+\epsilon}.
  \end{align}
\end{restatable}
Results of this type are known as \textit{minimax lower bounds}. Indeed, the result states that, for every policy, there exists at least one instance in which the expected regret grows at a certain rate. The bound is consistent with the known lower bounds for the HT and PS MAB problems. Indeed, in HT MABs every policy has its expected regret lower bounded by $\Omega(K^\frac{\epsilon}{1+\epsilon}T^\frac{1}{1+\epsilon})$ \citep{bubeck2013bandits}, while in PS MABs the lower bound is $\Omega(\sqrt{K\Upsilon T})$ \citep{garivier2011upper}. Thus, Equation \eqref{eq:regret_lb} is a natural combination of these two results that can be recovered by either setting $\Upsilon=1$ or $\epsilon = 1$. We refer to Appendix \ref{apx:proofs} for the proof.
\section{Technical Preliminaries}
\label{sec:technical_preliminaries}
In this section, we introduce the technical tools we employ in our proposed solution. First, we discuss the mean estimation for HT random variables and describe the \textit{Catoni estimator}. Then, we formalize the \textit{change-point detection} (CPD) problem and discuss a technique based on \textit{confidence sequences}.

\subsection{Mean Estimation for Heavy-Tailed Random Variables with Catoni Estimator}
Mean estimation for HT variables can be quite a delicate task. Empirical mean has been proved to achieve a sub-optimal concentration \citep{bubeck2013bandits}. However, alternative estimators enjoying optimal rates have been proposed. We focus on the elegant \textit{Catoni estimator} \citep{catoni2012challenging}, defined using a \textit{Catoni-type influence function} $\phi : \mathbb{R} \rightarrow \mathbb{R}$. The Catoni estimator $\widehat{\mu}_c$ for a sequence of variables $\{X_i\}_{i=1}^n$ is the solution of:
\begin{equation}
\label{eq:catoni_function}
\sum_{i=1}^n \phi_{\epsilon}(\lambda_i(X_i-\widehat{\mu}_c )) = 0,
 \quad \text{where} \quad \phi_{\epsilon}(x) = 
      \log\left(1+|x|+\frac{|x|^{1+\epsilon}}{1+\epsilon}\right),
\end{equation}
and $\{\lambda_i\}_{i=1}^n$ is a predictable process. Remarkably, for a proper choice of $\{\lambda_i\}_{i=1}^n$, this estimator enjoys an optimal concentration of order $\mathcal{O}\Big(\Big(\frac{v^\frac{1}{\epsilon}\log(\delta^{-1})}{n}\Big)^{\frac{\epsilon}{1+\epsilon}}\Big)$ with probability $1-\delta$ \cite{bhatt2022catoni}.
\subsection{Confidence Sequences and Change Point Detection}
The PS MAB is often addressed by resorting to \textit{change-point detection} (CPD) strategies, \emph{e.g.}, CUSUM-UCB \citep{liu2018change}. The idea is to actively adapt to environmental changes and tackle the problem as a sequence of stationary MABs. These strategies are often restricted to sub-gaussian rewards and do not scale on heavy-tailed variables. We propose an alternative approach to tackle this family of problems using a CPD strategy based on \textit{confidence sequences}. 

\textbf{Confidence Sequences.}~~Suppose $\{X_t\}_{t\in \mathbb{N}}\sim P$ for some $P \in \mathcal{P}^\mu$ where $\mathcal{P}^\mu$ is the set of distributions on $\prod_{t \in \mathbb{N}} \mathbb{R}$ such that $\mathbb{E}[X_t|\mathcal{F}_{t-1}]= \mu$ for each $t \in \mathbb{N}$, where $\mathcal{F}_{t-1}$ is the filtration.
A \textit{confidence sequence} (CS) for the mean is a sequence of confidence intervals $\{\text{CI}_t\}_{t \in \mathbb{N}}$ holding at arbitrary data-dependent stopping times. Formally:
\begin{equation}
    \label{eq:cs}
    \mathbb{P}(\forall t \in \mathbb{N}^+: \mu \in \text{CI}_t) \ge 1-\gamma.
\end{equation}
The random intervals $\{\text{CI}_t\}_{t \in \mathbb{N}^+}$ that satisfy property \eqref{eq:cs} are called $(1-\gamma)$-CS, where $1-\gamma$ is the confidence level. For a CS defined on $\mathbb{R}$, we formally introduce its \textit{width} after $t$ samples defined as $w(t, P, \gamma) \coloneqq 
\sup_{\mu_1,\mu_2 \in CI_t} |\mu_1-\mu_2| \le w(t, P, \gamma)
$,
for all $P \in \mathcal{P}^\mu$ and $\gamma \in (0,1]$.

\textbf{Change-Point Detection.}~~
Consider a data-generating process composed of infinitely-countable distributions $\{P_t\}_{t \in \mathbb{N}}$ and let $t_c \ge 1$ be an unknown breakpoint, \emph{i.e.}, $P_t = P_0$ for every $t \le t_c$ and $P_t = P_1$ for every $t > t_c$. The goal of a CPD algorithm is to detect, as soon as possible after $t_c$, that a change in the data-generating distribution happened. In other words, given the (stochastic) stopping time $\tau \in \mathbb{N}$ in which the CPD system detects a change, the objective is to minimize the \textit{detection delay }$\mathbb{E}_{t_c}[\tau-t_c ]$, where the expectation $\mathbb{E}_{t_c}$ is taken over an environment having a change-point after $t_c$ rounds. 
A trivial CPD system yielding a signal at every round would minimize this quantity. On the other hand, we also desire to reduce the \textit{false alarm rate} (FAR), \emph{i.e.}, the probability that a signal is produced when no change happens. This translates into minimizing $\mathbb{P}_\infty( \tau < \infty)$, where $\mathbb{P}_\infty$ is the probability measure of the environment where there is no change-point.
Moreover, the \textit{average run length} (ARL), defined as $\mathbb{E}_\infty[\tau]$, represents the expected number of rounds before a change is erroneously detected. However, when ARL is too large (\emph{e.g.}, the trivial CPD that never yields a signal), the system is too conservative, impacting the detection delay. This highlights a crucial trade-off between detection delay and ARL. 
Recent literature \citep{shekhar2023sequential,shekhar2023reducing} shed light on the possibility of reducing CPD to a sequential estimation, \emph{i.e.}, producing sequential testing via confidence sequences. 
We focus on the \texttt{repeated-FCS-detector} framework, introduced in \cite{shekhar2023reducing}. \texttt{repeated-FCS-detector} is a meta-algorithm that requires a CS computation strategy and uses it as a black-box tool, defined as:
\begin{definition}[\texttt{repeated-FCS-detector}, \cite{shekhar2023reducing}]
    Let $\{X_t\}_{t\in \mathbb{N}}$ be a sequence of observations. At every round $t$, we receive a new sample $X_t$ and initialize a new $(1-\gamma)$-CS for the mean $\text{CS}^{(t)}\coloneqq\{\text{CI}_n^{(t)}\}_{n \ge t}$, formed using samples $X_t$, $X_{t+1}$, $X_{t+2}$, and onwards. Moreover, we update all previously initialized CS $\{\text{CI}^{(i)}\}_{i < t}$ using $X_t$.
    We define the stopping time, $\tau$, as the first time at which the intersection of all initialized CS becomes empty, \emph{i.e.}, $\tau = \inf_{t \in \mathbb{N}}\{\bigcap_{n=0}^{t}  \text{CS}^{(n)} = \emptyset \}$.
\end{definition}
Provided an oracle capable of computing a $(1-\gamma)$-CS at every round, this strategy is distribution agnostic, as it requires no additional information about the data-generating distribution nor the change point. In \cite{shekhar2023reducing}, theoretical guarantees on both the ARL and the detection delay of \texttt{repeated-FCS-detector} are provided expressed in terms of the width of the $(1-\gamma)$-CS provided to the detector. 
We now report the theoretical guarantees of \texttt{repeated-FCS-detector}.
\begin{theorem}[Guarantees of \texttt{repeated-FCS-detector}, \cite{shekhar2023reducing}]
\label{thr:fcs-detector}
    Consider a CPD problem with observations $\{X_t\}_{t \in \mathbb{N}}$ i.i.d. from $P_0 \in \mathcal{P}^{\mu_0}$ for $t \le t_c$  and from $P_1 \in \mathcal{P}^{\mu_1}$ for $t > t_c$. Let $\delta \coloneqq |\mu_1 -\mu_0|$. Suppose we can construct $(1-\gamma)$-confidence sequences with pointwise width $w(t, P_0, \gamma)$ and $w(t, P_1, \gamma)$ for pre- and post-change mean, respectively. Then, we have: \textbf{(i)} When there is no changepoint, the \texttt{repeated-FCS-detector} satisfies $\mathbb{E}_\infty[\tau] \ge \frac{1}{\gamma}$; \textbf{(ii)} Suppose $t_c < \infty$ and large enough to ensure that $w(t_c, \mu_0, \gamma) < \delta$. Introduce the event $\mathcal{E} = \{\mu_0 \in \bigcap_{t=1}^{t_c} CI_t^{(1)} \}$, and note that $\mathbb{P}(\mathcal{E})\ge 1-\gamma$ by construction. Then, for $\gamma \in (0,0.5)$, we have $\mathbb{E}_{t_c}[(\tau- t_c)^+|\mathcal{E}] \le \frac{3u_0(\mu_0,\mu_1,t_c)}{1-\gamma}$, where  $u_0(\mu_0,\mu_1,t_c) \coloneqq \min_{n \in \mathbb{N}}\{w(n, \mu_1, \gamma)+w(t_c, \mu_0, \gamma) < \delta\}$.
\end{theorem}
Point (i) provides a lower bound on the ARL of \texttt{repeated-FCS-detector}, while (ii) upper bounds the expected detection delay. While ARL only depends on the desired confidence level, the detection delay depends on the width of the CSs. Indeed, the width of the CS must decrease fast enough to make the change detectable. This assumption is standard in CPD, as enough samples before the change are needed to model the null hypothesis correctly.
\section{Robust Regret Minimization in Piecewise-Stationary Heavy-Tailed Bandits}
\label{sec:algo}
In this section, we describe our strategy for regret minimization in HTPS MABs. We start by providing a CPD strategy suited for heavy-tailed random variables together with its theoretical guarantees.
Then, we leverage this tool to build a meta-algorithm named \algname, which uses a regret minimizer for the stationary setting and the CPD strategy to tackle non-stationarity.
\subsection{\texttt{Catoni-FCS-detector}}
We start by introducing a novel CPD strategy for HT random variables, which we name \texttt{Catoni-FCS-detector}, based on a \texttt{repeated-FCS-detector} using a special type of CS. We can define \texttt{Catoni-FCS-detector} as a special instantiation of \texttt{repeated-FCS-detector}.
\begin{definition}[\texttt{Catoni-FCS-detector}]
    An instance of \texttt{repeated-FCS-detector} is a \texttt{Catoni-FCS-detector} if the $(1-\gamma)$-CS $\{\text{CI}_t^\phi\}_{t \in \mathbb{N}}$ is defined as:
    \begin{align}
    \label{eq:catoni_cs}
        &CI_t^\phi = \bigg\{ m \in \mathbb{R} : \sum_{i=1}^t \phi_{\epsilon}(\lambda_i(X_i-m)) \in \bigg[ \mp \frac{v}{2} \sum_{i=1}^t\lambda_i^{1+\epsilon}\pm\log \bigg(\frac{2}{\gamma}\bigg)\bigg]\bigg\},
    \end{align}
    where $\phi_\epsilon$ is the Catoni-type influence function (Equation \ref{eq:catoni_function}).
\end{definition}
From now on, we call \textit{Catoni CS} the confidence sequence defined as in Equation \eqref{eq:catoni_cs}. Catoni CS have been introduced for the first time in \cite{wang2023catoni}. While a Catoni CS does not admit a trivial closed-form representation, it can be proven (see Appendix \ref{apx:proofs}) that Equation \eqref{eq:catoni_cs} represents a proper $(1-\gamma)$-CS for the mean, attaining an optimal width \citep{bhatt2022catoni}. To use \texttt{Catoni-FCS-detector} in a bandit problem, however, we need specific types of guarantees, different than the ones provided for the general \texttt{repeated-FCS-detector} framework. We now provide two novel contributions of independent interest. First, we show how the width of the Catoni CS can be narrowed further w.r.t. to the one presented in previous works in the case of infinite variance. Second, we provide a finite-time bound on the detection delay of \texttt{Catoni-FCS-detector}, a crucial property for using a CPD in a bandit.
\begin{restatable}[Detection Delay of \texttt{Catoni-FCS-detector}]{proposition}{detectionDelay}\label{prop:detection_delay}
\label{prop:detection_delay}
    Consider a CPD problem with observations $\{X_t\}_{t\in \mathbb{N}}$ drawn i.i.d. from $P_0 \in \mathcal{H}_{\epsilon, v} \cap \mathcal{P}^{\mu_0}$ for $t \le t_c$  and from $P_1 \in \mathcal{H}_{\epsilon, v} \cap \mathcal{P}^{\mu_1}$ for $t > t_c$. Let $\delta \coloneqq |\mu_1 -\mu_0|$. Suppose that there exists a known upper bound $T$ of the change point ($t_c \le T$). Let $n_{min} \coloneqq 68\log(T^\frac{1+\epsilon}{\epsilon})$ and suppose $t_c\ge n_{min}$ large enough s.t. $w(t_c,P_0, \gamma)\le \frac{\delta}{2}$. Set $\gamma = \frac{2}{T^3}$. Then, there exists a predictable sequence $\{\lambda_i\}_{i=1}^T$ s.t. \texttt{Catoni-FCS-detector} enjoys \textbf{(i)} $\mathbb{P}_{t_c}\left((\tau- t_c)^+ \le \mathcal{O}\left(v^\frac{1}{\epsilon}\frac{\log (T)}{\delta^{\frac{1+\epsilon}{\epsilon}}}\right)\right)\ge 1-\frac{14}{T}$ and \textbf{(ii)} $\mathbb{P}_{t_c}\left(\tau < t_c\right) \le \frac{14}{T}$.
\end{restatable}
We point out the importance of this specialized result. Since the guarantees of Theorem \ref{thr:fcs-detector} are very general, this result is aimed at providing a finite-time, high-probability bound on the detection delay when using Catoni CS. In particular, we make the term $u_0(\mu_0, \mu_1, t_c)$ from Theorem \ref{thr:fcs-detector} explicit by using the properties of Catoni CS. Due to space reasons, the proof is postponed to Appendix \ref{apx:proofs}. 
Note that the rate of this detection delay cannot be improved as the lower bound for the detection delay of any distribution change is $\Omega(\log(\gamma^{-1}))$ \citep{lorden1971procedures}, where $\gamma$ the confidence parameter that we set to $\mathcal{O}(1/T)$. Moreover, the dependencies on $\delta$, $\epsilon$ and $v$ may also be tight, as they embed the log-likelihood ratio of the test for the means of heavy-tailed random variables. We leave the answer to this question for further investigations. We conclude this section with two important remarks.
\begin{remark}[Comparison with Existing CPDs]
    \texttt{Catoni-FCS-detector} is, to the best of authors' knowledge, the first CPD strategy for the mean of HT random variables with infinite variance enjoying such guarantees. Thus, we consider our analysis an interesting standalone contribution. In bandit literature, however, many CPD strategies have been employed, (\emph{e.g.}, CUSUM \citep{liu2018change} and GLR Test \citep{besson2022efficient}). However, they do not cover the HT scenario and often rely on strong parametric assumptions on the sample-generating distribution, \emph{e.g.}, only working on Bernoulli variables.
\end{remark}
\begin{remark}[On the a priori knowledge of \texttt{Catoni-FCS-detector}]
    \texttt{Catoni-FCS-detector} does not rely, in principle, on any prior knowledge of the magnitude of the change or on the means. The confidence parameter $\gamma$ is set based on the time horizon $T$, which is standard in MABs. Moreover, the sequence $\{\lambda_i\}_{i=1}^t$ can be set in advance for every $t \in [T]$, only relying on the knowledge of $T$. 
\end{remark}
\subsection{\algname}
\begin{wrapfigure}{R}{.5\textwidth}
\SetInd{0.2em}{0.2em}
\vspace{-.5cm}
\begin{minipage}{.5\textwidth}
\RestyleAlgo{ruled}
\LinesNumbered
\begin{algorithm}[H]
\SetAlgoNoEnd   
\SetAlgoNoLine  
\DontPrintSemicolon
\caption{\algname}\label{alg:alg}
\SetKwInOut{Input}{Input}
\small
\Input{Number of actions $K$, time horizon $T$, uniform exploration $\eta$, a policy $\pi_s$.}

Initialize $ t \leftarrow 0$, $t' \leftarrow 0$, $N_{i, t} \leftarrow 0~~~ \forall i \in [K]$.

Set $\gamma \leftarrow \frac{2}{T^3}$.

\For{$t \in [T]$}{
\If{$t'~\text{mod} \left\lfloor {K}/{\eta}\right\rfloor \le K$}{
Select and play $I_t \leftarrow t'~~\text{mod} \left\lfloor {K}/{\eta}\right\rfloor$.
}
\Else{

Update $\pi_s$ with the history of the last $t'$ rounds.

Select and play $I_t$ according to $\pi_s$.

}

Receive $X_t$ and update $N_{I_t, t} \leftarrow N_{I_t, t} + 1$ and $t' \leftarrow t' + 1$.

\If{$N_{I_t, t} \ge n_{min}$}{
    Start a new $\left(1-\gamma\right)$-CS $\text{CS}_{I_t}^{(t')}$ for action $I_t$, according to Equation \eqref{eq:catoni_cs}.

    \If{$\exists a,b \in [t'] : \text{CS}_{I_t}^{(a)} \cap  \text{CS}_{I_t}^{(b)} = \emptyset$}{
    Reset  $t' \leftarrow 0$, $N_{i, t} \leftarrow 0~~~ \forall i \in [K]$.
        
    Remove all initialized CS.
    }
}
}
\end{algorithm}
\end{minipage}
\vspace{-.4cm}
\end{wrapfigure}
In this section, we introduce \algname (\algnameshort for short, Algorithm \ref{alg:alg}), an algorithm for PS HT bandits. \algnameshort actively adapts to the changes in the reward-generating distribution. The algorithm has three components: $(1)$ a sub-algorithm suited for the stationary HT MAB problem, that aims to minimize the regret in the stationary segments, we call this policy $\pi_s$; $(2)$ the \texttt{Catoni-FCS-detector} strategy for CPD; and $(3)$ a cyclic uniform exploration that ensures the availability of enough samples for every action to perform the CPD test.
Algorithm \ref{alg:alg} proceeds as follows: roughly every $\lfloor K/\eta \rfloor$ rounds it tries all the actions once (lines 4-5), this ensures that CPD can happen efficiently even when underrepresented actions in the history are the only ones changing. In the other rounds, a bandit sub-routine (\emph{e.g.}, \texttt{Robust-UCB}) plays according to all the history since the last reset (lines 8-9); once the new reward is obtained, it is fed to the \texttt{Catoni-FCS-detector}  that verifies if a change point occurred (lines 12-14), in this case, everything is reset (line 15-16). 
\begin{remark}(Connection to \texttt{Monitored-UCB} from \cite{cao2019nearly})
\algname borrows the idea of cyclic uniform exploration from the \texttt{Monitored-UCB} \cite{cao2019nearly}. Moreover, as most of the algorithms for the PS setting, ours share the usage of a stationary bandit sub-routine. However, a crucial difference relies on the type of CPD strategy employed. \texttt{Monitored-UCB} leverages a sliding-window type of CPD strategy that checks if the average of the first half of the sliding-window is significantly different from that of the second half. This type of CPD strategy requires two hyper-parameters, the window size and the threshold, respectively. Tuning these parameters may be difficult, even though, in practice, the algorithm works well even under misspecification.  
Finally, \texttt{Monitored-UCB} only deals with rewards bounded in $[0,1]$, while \algname deals with HT rewards.
\end{remark}
\begin{remark}(Priori Knowledge of \algnameshort)
Algorithm \ref{alg:alg} receives as inputs the time horizon $T$, the uniform exploration coefficient $\eta$, and a regret minimizer for the stationary setting ${\pi_s}$ only. Assuming that it is possible to choose a regret minimizer that does not require additional parameters other than $T$ (which is, as we will show in the next section, rather natural), then the only knowledge that \algnameshort requires on the environment is the time horizon $T$. 
Thus, in principle, our algorithm requires knowledge of $T$ only. In practice, this property ensures that no tuning must happen.
\end{remark}

\subsection{Theoretical Guarantees of \algname}
As customary in the literature of PS MABs, we introduce a technical assumption regarding the length of any epoch ensuring that exploration is frequent enough to detect for every action.
\begin{assumption}
\label{ass:length}
    For every epoch $j \in [\Upsilon]$, let $\widetilde{\delta}_{min}^{(j)} \coloneqq \min\{\delta_{min}^{(j-1)}, \delta_{min}^{(j)}\}$,
    and let $|E_j|$ be its length and $L_j \coloneqq 6 (236)^\frac{1+\epsilon}{\epsilon}v^\frac{1}{\epsilon}\frac{\log\left(\log\left(1/{{\widetilde{\delta}_{min}^{(j)
}}}\right)\right)+\log(T)}{(\widetilde{\delta}_{min}^{(j)
})^{\frac{1+\epsilon}{\epsilon}}}$.
    The learner can select $\eta$ such that, for every $j \in [\Upsilon]$, it holds that $|E_j|\ge 2 n_{min} + 2\left\lceil L_j {K}/{\eta}\right\rceil$.
\end{assumption}
This assumption ensures that proper learning can be performed in such an environment
Indeed, we enforce that every epoch $j \in [\Upsilon]$ is large enough so that, due to the forced exploration only, the algorithm chooses every action at least $L_j$ times. This kind of assumption is ubiquitous in the piecewise-stationary bandits literature. Notable examples include Assumptions 4 and 7 in \citep{besson2022efficient}, Assumptions 1 and 2 in \citep{cao2019nearly} and Assumption 1 in \citep{liu2018change}. Some are equivalent to ours, while others are neither weaker nor stronger. Alternative assumptions, such as the monotonicity of the mean change, also allow for theoretical tractability, \emph{e.g.}, Assumption 1 in \citep{seznec2020single}, which forces expected rewards to evolve in a decreasing manner. Note that Assumption \ref{ass:length} is a technical assumption aimed at the theoretical analysis of the algorithm. Algorithm \ref{alg:alg} can operate regardless of this assumption, as shown in Section \ref{sec:experiments}. We are now ready to present our main result.
\begin{restatable}[Regret Upper Bound of \texttt{R-CPD-UCB}]{theorem}{regretBound}
  \label{thr:ub}  
  Under Assumption \ref{ass:length}, \algnameshort suffers an expected cumulative regret bounded as:
  {\thinmuskip=1mu
\medmuskip=1mu \thickmuskip=1mu
  \begin{align}
      &\mathbb{E}[R^{\pi^{\text{\emph{\algnameshort}}}}(T)] \le \mathcal{O}\Bigg(  \underbrace{\sum_{j=1}^{\Upsilon} \frac{v^\frac{1}{\epsilon}\log(T)}{{(\widetilde{\delta}_{min}^{(j)
})}^{\frac{1+\epsilon}{\epsilon}}}\left\lceil\frac{K}{\eta}\right\rceil\Delta_{max}^{(j)}}_{\text{\emph{\textcolor{blue}{(\texttt{A}) Detection Delay Contribution}}}} + \underbrace{\sum_{j=1}^{\Upsilon}\mathbb{E}[R^{{\pi_s}}(|E_j|)]}_{\text{\emph{\textcolor{red}{(\texttt{B}) Stationary Policy Regret}}}} + \underbrace{\eta \sum_{j=1}^\Upsilon |E_j|\Delta_{max}^{(j)}}_{\text{\emph{\textcolor{green!50!black}{(\texttt{C}) Uniform Exploration}}}} \Bigg). \label{eq:regret_ub_cpd_ucb}
  \end{align}}
\end{restatable}
The regret can be decomposed into three  contributions due to: the detection delay (part \textcolor{blue}{(A)}), the regret-per-epoch of the stationary policy (part \textcolor{red}{(B)}), and the rounds of uniform exploration (part \textcolor{green!50!black}{(C)}).

\textbf{Uniform Exploration Trade-off.}~~The uniform exploration parameter $\eta$ appears in both \textcolor{blue}{(A)} and \textcolor{green!50!black}{(C)}. Setting aside part \textcolor{red}{(B)}, it is clear that $\eta$ creates a trade-off between these two: the larger $\eta$ is, the quicker the algorithm can detect a change, and the smaller is \textcolor{blue}{(A)}; on the other hand, excessive uniform exploration inflates the regret of \algnameshort and the contribution from  \textcolor{green!50!black}{(C)}. Finding the optimal value for $\eta$ would require extensive prior knowledge, which is, in general, not available. A good trade-off is to set $\eta = \sqrt{\Upsilon/T}$, which impose both \textcolor{blue}{(A)} and \textcolor{green!50!black}{(C)} to be $\widetilde{\mathcal{O}}(\sqrt{\Upsilon T})$. However, it is possible to define a forced exploration strategy that does not require any knowledge of $\Upsilon$, making the algorithm more versatile while keeping the same order of performance. In particular, we can leverage the methodology developed in \cite{besson2022efficient} and obtain the following result.
\begin{restatable}{corollary}{choosingeta}
  \label{thr:choosingeta} 
  Let $\{\eta_j\}_{j\in \mathbb{N}}$ where $\eta_j = \eta_0 \sqrt{jK\log(T)/T}$ for some $\eta_0 > 0$ be an increasing sequence. \algnameshort using $\eta_{j+1}$ after the $j$-th detection satisfies:
  \begin{align}
  \text{\textcolor{blue}{(A)}} &\le \frac{v^\frac{1}{\epsilon}\sqrt{K\Upsilon T \log(T)}}{\eta_0 {{\delta}_{min}}^{\frac{1+\epsilon}{\epsilon}}}\Delta_{max},  \qquad 
  \text{\textcolor{green!50!black}{(C)}} \le \eta_0\sqrt{K(\Upsilon+1) T \log(T)}\Delta_{max}. \label{eq:choosingeta_2}
  \end{align}
\end{restatable}
Note that, if ${\delta}_{min}$ is known, setting $\eta_0 = \delta_{min}^{-\frac{1+\epsilon}{2\epsilon}}$ can further reduce the regret bound. 

\textbf{Choosing $\pi_s$.}~~ 
The choice of the inner regret minimizer $\pi_s$ determines the magnitude of part \textcolor{red}{(B)}. The best choice is to select a policy that has a regret upper bound matching the known lower bound of $\Omega(K^\frac{\epsilon}{1+\epsilon}T^\frac{1}{1+\epsilon})$. We can instantiate \algnameshort using the \texttt{Robust UCB} policy with \textit{median-of-means estimator} from (\citealp{bubeck2013bandits}, Section 2.2). As a result, we get the following bounds.
{\thinmuskip=1mu
\medmuskip=1mu 
\thickmuskip=1mu\begin{restatable}{corollary}{regretBoundInstance}
  \label{thr:ub_mom} 
    Let $\pi_s$ be the \texttt{Robust UCB} policy with median-of-means estimator (\citealp{bubeck2013bandits}, Section 2.2). Under Assumption \ref{ass:length}, \algnameshort suffers an expected cumulative regret bounded as:
    \begin{align}
      \mathbb{E}[&R^{\pi^{\text{\emph{\algnameshort}}}}(T)] \le \mathcal{O}\Bigg( \text{\emph{\textcolor{blue}{(\texttt{A})}}} + \underbrace{\sum_{j=1}^{\Upsilon}\sum_{i:\Delta_i^{(j)}>0} \frac{v^\frac{1}{\epsilon}\log \left(|E_j|\right)}{(\Delta_i^{(j)})^\frac{1}{\epsilon}}}_{{\text{\makebox[0pt]{\emph{\textcolor{red}{(\texttt{B}$_1$) Robust UCB Regret}}  \text{\emph{\textcolor{red}{(Instance Dependent)}}} }}}} + \text{ \emph{\textcolor{green!50!black}{(\texttt{C})}}} \Bigg). \label{eq:regret_mom}
  \end{align}
  Moreover, if $\log(|E_j|) \ge \frac{5(\Delta_{max}^{(j)})^\frac{1+\epsilon}{\epsilon}}{2v^\frac{1}{\epsilon}}$ for every $j\in[\Upsilon]$, we have:
  \begin{equation}\label{eq:regret_instance_indep}
      \mathbb{E}[R^{\pi^{\text{\emph{\algnameshort}}}}(T)] \le \widetilde{\mathcal{O}}( \text{\emph{\textcolor{blue}{(\texttt{A})}}} + \underbrace{\textcolor{white}{\Bigg|}(K\Upsilon)^\frac{\epsilon}{1+\epsilon}(vT)^\frac{1}{1+\epsilon}}_{{\text{\makebox[0pt]{\emph{\textcolor{red}{(\texttt{B}$_2$) Robust UCB Regret}}  \text{\emph{\textcolor{red}{(Instance Independent)}}} }}}} + \text{ \emph{\textcolor{green!50!black}{(\texttt{C})}}} ).
  \end{equation}
\end{restatable}}
Equation \eqref{eq:regret_mom} is a direct consequence of Theorem \ref{thr:ub} and Theorem 3 of \cite{bubeck2013bandits}. Equation \eqref{eq:regret_instance_indep} follows from Theorem \ref{thr:ub}, Proposition 1 of \cite{bubeck2013bandits}, and Jensen's inequality. \texttt{Robust UCB} enjoys both instance-dependent and instance-independent guarantees: part \textcolor{red}{(B$_1$)} depends on the sub-optimality gaps $\Delta_i^{(j)}$ and the individual lengths of the epochs, while part \textcolor{red}{(B$_2$)} does not, as it accounts for a worst-case scenario of the sub-optimality gaps. 
We can now combine all and get the following. 
\begin{restatable}{corollary}{regretBoundFinal}
  \label{thr:ub_final} 
    Let $\pi_s$ be the \texttt{Robust UCB} policy with median-of-means estimator from (\citealp{bubeck2013bandits}, Section 2.2). Let $\{\eta_j\}_{j\in \mathbb{N}}$ where $\eta_j = \eta_0 \sqrt{jK\log(T)/T}$ for some $\eta_0 > 0$.  Under Assumption \ref{ass:length}, \algnameshort using $\eta_{j+1}$ after the $j$-th detection suffers an expected cumulative regret bounded as:
    \begin{align}
      \mathbb{E}[&R^{\pi^{\text{\emph{\algnameshort}}}}(T)] \le  \mathcal{O}\Bigg(  \frac{v^\frac{1}{\epsilon}\sqrt{K\Upsilon T \log(T)}}{\eta_0 {{\delta}_{min}}^{\frac{1+\epsilon}{\epsilon}}}\Delta_{max}+ \frac{K\Upsilon v^\frac{1}{\epsilon}\log \left(T/\Upsilon\right)}{\Delta_{min}^\frac{1}{\epsilon}} \Bigg). \label{eq:final_regret_id}
  \end{align}%
  Moreover, if $\log(|E_j|) \ge 3(\Delta_{max}^{(j)})^\frac{1+\epsilon}{\epsilon}v^{-\frac{1}{\epsilon}}$ for every $j\in[\Upsilon]$, and $\delta_{min}^\frac{1+\epsilon}{\epsilon} \ge v^\frac{1}{\epsilon(1+\epsilon)}(\Upsilon K / T)^\frac{1-\epsilon}{2(1+\epsilon)}\sqrt{\log(T)}$, we have:
  \begin{align}
      &\mathbb{E}[R^{\pi^{\text{\emph{\algnameshort}}}}(T)] \le \widetilde{\mathcal{O}}( \textcolor{white}{\Bigg|}(K\Upsilon)^\frac{\epsilon}{1+\epsilon}(vT)^\frac{1}{1+\epsilon} ). \label{eq:final_regret}
  \end{align}
\end{restatable}
\vspace{-.3cm}
Equation \eqref{eq:final_regret_id}, depends on both the minimum mean change $\delta_{min}$, and the extreme sub-optimality gaps $\Delta_{min}$ and $\Delta_{max}$, along the whole trial. We consider this bound an instance-dependent guarantee over the performance of \algnameshort. Equation \eqref{eq:final_regret}, instead, does not contain any of these quantities. The second assumption fundamentally states that $\delta_{min}$ can be assumed to be a constant w.r.t. the other quantities, in particular $T$. In this case, an instance-independent bound can be obtained. Equation \eqref{eq:final_regret} matches, up to constants, the lower bound presented in Theorem \ref{thr:lb}. Thus, if we focus on the dependence on $T$, $\Upsilon$, $v$, and $K$ the performance guarantees of \algnameshort are nearly-optimal.
\section{Numerical Evaluation}
\label{sec:experiments} We now provide a numerical evaluation of \algname ($\pi_s$ chosen as \texttt{Robust UCB} with median-of-means estimator). We refer to Appendix \ref{apx:exp} for additional details and experimental campaigns.
\subsection{Casting Real-World Data to HTPS MABs}
We model a real-world scenario as an HTPS MAB and, then, we leverage a real dataset to generate an HTPS MAB instance on which \algnameshort is tested. 
~\textbf{Setting.}~~We consider the problem of profit maximization in financial trading. As pointed out by \cite{panahi2016model}, financial data exhibit heavy tails. A financial application of HT MABs is identifying the most profitable cryptocurrency among $K$ options. In fact, at the start of each day, an investor would like to invest a share of their money in the cryptocurrency with the highest closing price. This very same application has been studied, for example, in \cite{yu2018pure} and \cite{lee2022minimax}, both in the context of HT MABs. We use the same dataset (\href{https://www.kaggle.com/datasets/sudalairajkumar/cryptocurrencypricehistory?select=coin_Ethereum.csv}{Kaggle link}) employed in \cite{lee2022minimax}. In Figure \ref{fig:crypto_fit}, we report the closing prices of four selected currencies among the top ten by market capitalization, along with a piecewise-constant fit of the data that minimizes the squared error. We observe two things: first, the piecewise constant approximation is a better fit than any constant approximation (in the previous works on HT MABs, the reward from a given currency was always considered stationary); second, this approximation suffers a high error in certain segments where the stochastic fluctuations are really strong. Following the existing literature \cite{panahi2016model}, we fit the price distribution inside every segment with a Pareto distribution having its mean centered on the segment height, using $\epsilon < 1$ and $v=3$. Thus, the profit maximization problem in cryptocurrency trading can be treated as an HTPS MAB. ~\textbf{Results.}~~Starting from the piecewise-constant fit of the prices, we can build an HTPS MAB environment on which we test \algnameshort, together with \texttt{Sliding Window UCB} \cite{garivier2011upper} and \texttt{MR-APE} \cite{lee2022minimax}, which was already tested on the same dataset when assuming stationarity. In Figure \ref{fig:crypto_regret}, we report the cumulative regrets obtained by the three algorithms averaged over $20$ trials ($x$-axis is rescaled). \algnameshort performs better than the two competitors, as it is the only algorithm tackling both heavy-tailedness and non-stationarity of the setting.
\begin{figure}[t]
    \centering
    \includegraphics[width=\linewidth]{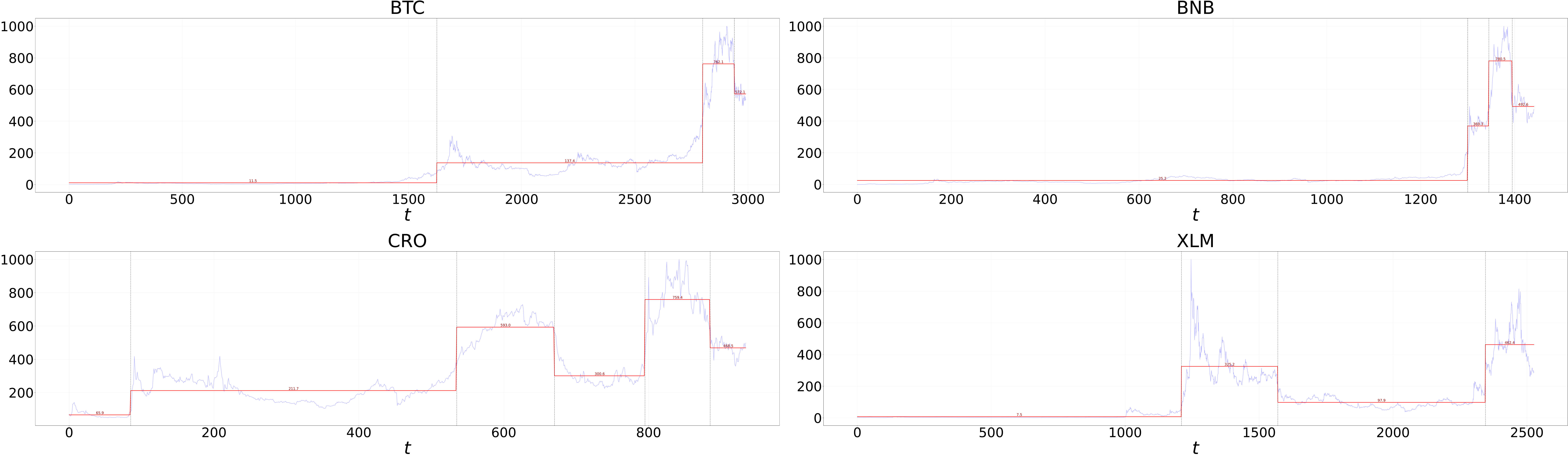}
    \caption{{Rescaled closing prices of four selected cryptocurrencies (blue) with a piecewise-constant approximation (red). Each time step is a day starting in April $2016$. Source: \href{https://www.kaggle.com/datasets/sudalairajkumar/cryptocurrencypricehistory?select=coin_Ethereum.csv}{Kaggle Dataset}.}}
    \label{fig:crypto_fit}
\end{figure}

\begin{figure}[!t]
  \centering
  \raisebox{.57cm}[0pt][0pt]{%
  \begin{minipage}[t]{0.32\textwidth}
    \centering
    \includegraphics[width=\linewidth]{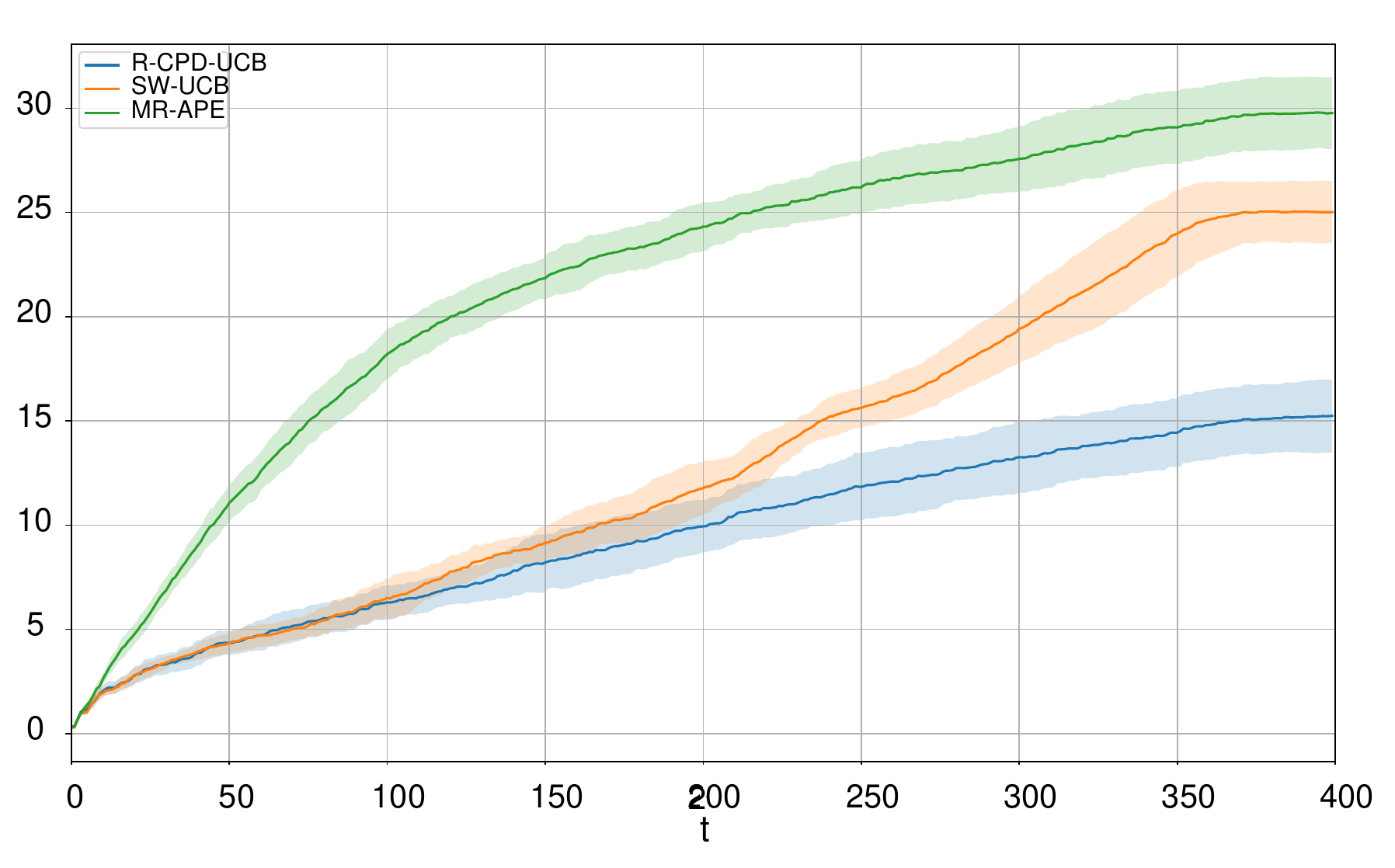} 
    \captionof{figure}{{Cumulative regrets on HTPS built from cryptocurrency dataset. $20$ trials, mean $\pm$ std.}}
    \label{fig:crypto_regret}
  \end{minipage}}
  \hfill
  \begin{minipage}[t]{0.64\textwidth}
    \centering
      \begin{subfigure}[t]{0.48\textwidth}
        \centering
        \includegraphics[width=\linewidth]{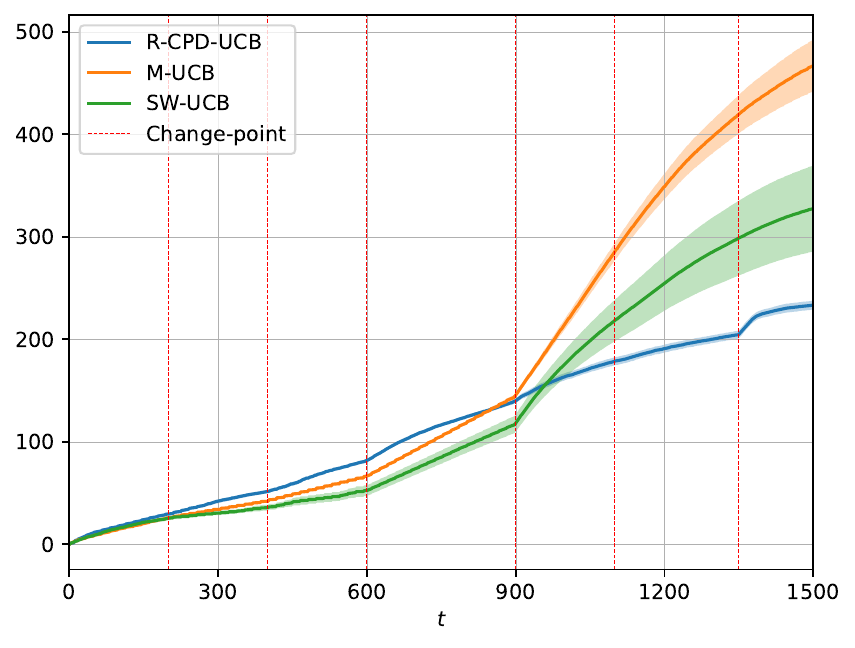}
        \caption{Gaussian rewards.}
        \label{fig:subA}
      \end{subfigure}
      \hfill
      \begin{subfigure}[t]{0.48\textwidth}
        \centering
        \includegraphics[width=\linewidth]{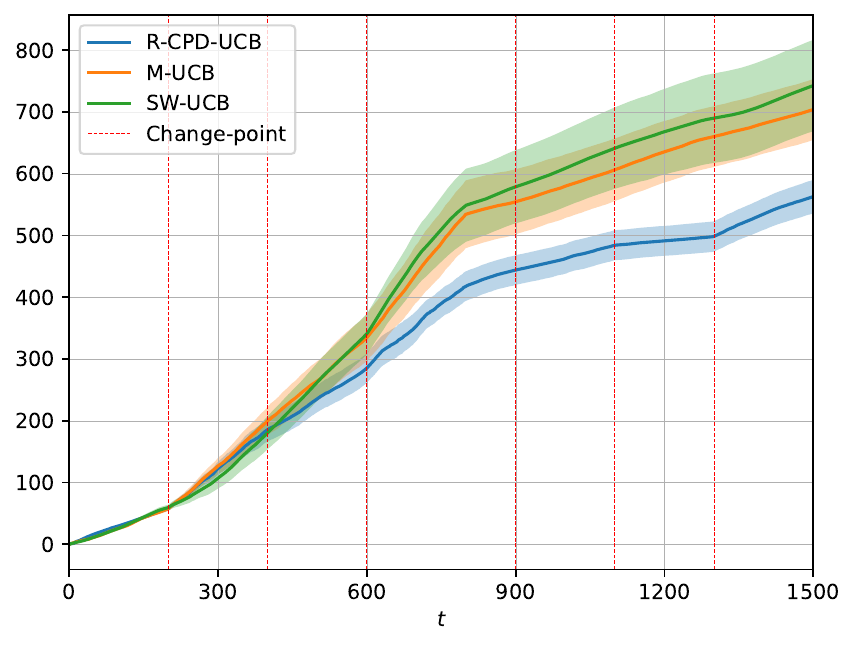}
        \caption{Pareto rewards.}
        \label{fig:subB}
      \end{subfigure}
    \captionof{figure}{{Cumulative regrets. $20$ trials, mean $\pm$ std.}}
    \label{fig:composite}
  \end{minipage}
\end{figure}

\subsection{Regret Minimization in Highly Non-Stationary Environments}
We now evaluate how \algnameshort behaves in highly dynamic scenarios where changes are close. \\ 
\textbf{Setting.}~~We compare \algnameshort with two of the most popular algorithms from the literature, \texttt{Monitored UCB} \cite{cao2019nearly} and \texttt{Sliding Window UCB} \cite{garivier2011upper}. We consider two PS MABs: Gaussian rewards with $\sigma=1$ and Pareto rewards with $\epsilon < \frac{1}{2}$. In both MABs, we have $K=3$, $T=1500$ and $\Upsilon = 6$. Also, we have $\delta_{min}=0$, \emph{i.e.}, and some actions may not change their means after a change point. However, at least one arm has its mean change, and the optimal action changes at least $4$ times. Interestingly, these instances violate Assumption \ref{ass:length}. ~\textbf{Results.}~~In Figure \ref{fig:results_high}, we report the cumulative regrets suffered by the considered algorithms. \algnameshort achieves, in both instances, a smaller cumulative regret than competitors. Moreover, it shows a smaller uncertainty and more stable performances across the trials, especially when rewards have infinite variance (Figure \ref{fig:pareto}). Interestingly, \algnameshort can outperform both \texttt{Monitored UCB} and \texttt{Sliding Window UCB} even when the rewards are Gaussian. This is because the change points are frequent and very close. Robust mean estimation using median-of-means stabilizes the algorithm's behavior in data-scarce regimes. Finally, we remark that Assumption \ref{ass:length} is violated by these two instances; however, \algnameshort performs well (and so is \texttt{Monitored UCB}, which relies on a similar hypothesis). This phenomenon was already observed in \cite{cao2019nearly}, and shows how Assumption \ref{ass:length}, in practice, is not very limiting.

\section{Conclusions}
In this work, we provided the first study on regret minimization in heavy-tailed piecewise-stationary bandits. We provided a lower bound on the performance of every algorithm and proposed \algname, a novel algorithm whose regret nearly matches the lower bound. We leverage novel advancements in the theory of change-point detection, building \texttt{Catoni-FCS-detector}, a general detection strategy suited for distributions with infinite variance. Finally, numerical evaluation shows that the performance of \algnameshort is solid when compared to existing baselines. An interesting future direction would be to study the HTPS MAB problem when $v$ and $\epsilon$ are unknown.

\clearpage
\bibliographystyle{plain}
\bibliography{bibliography}

\newpage
\appendix
\onecolumn
\section{Additional Related Works on Non-Stationary MABs}
\label{apx:related}
In this appendix, we discuss more in detail the related works on non-stationary MABs. 

\subsection{Piecewise-Stationary MABs}
The most common definition of piecewise-stationary MABs in the literature is the one introduced by \cite{yu2009piecewise}. In this work, the authors deal with the PS MAB problem as it is defined in this work, and consider both a scenario in which side-observations are available, and an agnostic scenario in which they are not, which corresponds to the one that we study in this work. In the latter scenario, they show that every algorithm must suffer at least $\Omega(\sqrt{T})$ regret. In \cite{garivier2011upper}, the authors analyze two algorithms to tackle the PS MAB problem, namely \texttt{Discounted UCB} (introduced in \cite{kocsis2006discounted}) and \texttt{Sliding Window UCB}. Contrary to ours, these algorithms don't rely on any CPD strategy but rather \textit{passively} adapt to the changes in the environment. In practice, \textit{actively} adapting algorithm, \emph{e.g.}, algorithm based on CPD strategies like ours, hence show better performances. The idea of actively adapt to changes first appeared in \cite{hartland2007change}. More recent works, such as \cite{liu2018change} and \cite{cao2019nearly}, paved the way for the analysis of actively adaptive algorithms, which were considered tougher to analyze from a theoretical perspective w.r.t. to their passive counterparts. Recently, with \cite{auer2019adaptively} and \cite{besson2022efficient}, there has been focus on removing the prior knowledge on $\Upsilon$ from the algorithms. In the former, the \texttt{AdSwitch} algorithm they propose does not require any additional assumptions, but it is not optimized for tractability or numerical efficiency. Indeed, as shown in \cite{besson2022efficient}, \texttt{AdSwitch} enjoys poor empirical performance.
In the latter, the authors propose an algorithm that performs well in practice and has tight theoretical guarantees without any need for $\Upsilon$ to be known beforehand, however they rely on an assumption which is nearly the same as ours. All of the aforementioned works don't account for the heavy-tailed setting, as their scope is restricted to rewards with bounded support or sub-gaussian. The only work that accounts for non-stationarity in heavy-tailed settings is \cite{bhatt2023piecewise}, where the authors consider a general framework to allow for more general risk measures (linear being the case considered here), and consider the same setup of piecewise-stationary bandits and heavy-tailed rewards, and establish upper and lower bounds on the regret under special assumptions on the risk measures and distributions. However, there are multiple reasons for why our approach is better suited in the regret-minimization scenario:
\begin{itemize}
    \item Assumption 1 (Stability): since their paper focuses on achieving strong regret guarantees with heavy-tailed rewards, the stability assumption plays a crucial role in the analysis, where the rate functions on the decay of the empirical and truncated distributions are assumed to be known. While the assumption is not strong in and of itself, the knowledge of these parameters play a crucial role in the change detection and regret minimization procedures, and also appear in the regret bounds. \algname, on the other hand, requires no knowledge of such functions, relies on novel analysis of Catoni estimators that do not necessitate truncations, and is simpler to run online in practice.
    \item The CPD routine used in \cite{bhatt2023piecewise} is based on a sliding window method requiring specification of both widow size and threshold. \algname is based on the newly developed CPD method based on Catoni estimator that runs online only with the same assumption on the distributions.
    \item The regret minimization algorithm in \cite{bhatt2023piecewise} uses a data-driven truncation of the rewards that depends on the policy, and the knowledge of the decay rate functions to compute the exploration bias for the arm index. Our work, on the other hand, requires no such methods and uses a simple combination of the novel Catoni-CPD and any policy suited for the stationary heavy-tailed regret minimization.
\end{itemize} 
For the most common case of linear risk/ regret in mean, \algname establishes stronger guarantees with the CPD procedure requiring weaker assumptions. Finally, it is easier to implement owing to not requiring distributional knowledge or thresholds.

\subsection{Bounded Variation and Monotonically Non-stationary MABs}
Another setting of interest is non-stationary MABs with \textit{bounded variations}. In this setting, the rewards' distributions changes are less restricted, and the focus moves from the number of changes to the total amount of change $V_T$. In \cite{besbes2014stochastic}, the authors propose \texttt{Rexp3}, an algorithm that leverages tools from the adversarial MAB problem to deal with non-stationarity in stochastic settings. The regret upper bound that they provide is in the order of $\mathcal{O}(V_T^\frac{1}{3}T^\frac{2}{3})$. Over the last years, there has been increasing interest in \textit{monotonically non-stationary} MABs, \emph{i.e.}, non-stationary MABs where the mean rewards are only allowed to decrease (rotting bandits, \cite{seznec2019rotting, seznec2020single}) or to increase (rising bandits, \cite{ metelli2022stochastic}), some works focus on both settings \cite{heidari2016tight, genalti2024graph}. The monotonicity assumptions substitutes the need for piecewise-stationarity, as it is a strong enough assumption to allow for strong theoretical characterizations. In such settings, regret bounds depend in general on the total variation of the distributions' means and instance-dependent-type of results are common in this literature. Moreover, the additional structure added by this assumption, put the accent on the difference between \textit{restless} bandits (a proper non-stationary setting) and \textit{rested} bandits, where the evolution of rewards depends on learner's actions rather than just time.

\section{Proofs}
\label{apx:proofs}
\regretLowerBound*
\begin{proof}
    The proof of this theorem combines techniques from Lemma 5 of \cite{seznec2020single}, Theorem 4 of \cite{genalti2024varepsilon}, and Theorem 6 from \cite{garivier2019explore}.

    Consider the following prototype of reward distribution, defined for $y\in \left(0, 1 \right)$ and $\Delta \in \left(0, 1 \right)$:
    \begin{equation*}
        \rho_y = \left(1-v^{-\frac{1}{\epsilon}}y^{\frac{1+\epsilon}{\epsilon}}\right)\delta_0 + \left(v^{-\frac{1}{\epsilon}}y^{\frac{1+\epsilon}{\epsilon}}\right)\delta_{v^{\frac{1}{\epsilon}}\Delta^{-\frac{1}{\epsilon}}}.
    \end{equation*}
    It is easy to verify that $\rho_y \in \mathcal{H}_{(1,\epsilon)}$ for every $y \in [0, \Delta]$.
    
    Consider a set of of instances belonging to $\mathcal{B}_{(v, \epsilon, \Upsilon)}$ indexed by a vector $\boldsymbol{i}^* \in [K]^\Upsilon$ in a way that, for every $j \in [\Upsilon]$ and every $t \in E_j$, we have 
    \begin{equation*}
    \nu_i^{(j)} = \begin{cases}
        \rho_{2^\frac{\epsilon}{1+\epsilon}\Delta},~~~&\text{if}~i= i_j^*\\
        \rho_\Delta,~~~~~~&\text{if}~i\neq i_j^* \\ 
    \end{cases}\,.
    \end{equation*}
    It follows that $\mu_{i_j^*}^{(j)} - \mu_{i}^{(j)} = \Delta$ for every $j \in [\Upsilon]$ and $i \neq i_j^*$.
    Let $|E_j| = \frac{T}{\Upsilon}$ for every $j \in [\Upsilon]$, assuming w.l.o.g. that $T$ is divisible by $\Upsilon$. Thus, all epochs are of the same length. For every fixed policy $\pi$, we write the average expected regret among the instances indexed by $\boldsymbol{i}^*$:
    \begin{align}
    \frac{1}{K^\Upsilon}\sum_{\boldsymbol{i}^* \in [K]^\Upsilon}\mathbb{E}_{\boldsymbol{i}^*}\left[R^\pi(T)\right] &= \frac{1}{K^\Upsilon}\sum_{\boldsymbol{i}^* \in [K]^\Upsilon}\sum_{j=1}^\Upsilon \Delta\mathbb{E}_{\boldsymbol{i}^*}\left[|E_j|-N_{i_j^*}^{(j)}\right]\notag\\
    &= \Delta \left(T - \frac{1}{K^\Upsilon}\sum_{\boldsymbol{i}^* \in [K]^\Upsilon}\sum_{j=1}^\Upsilon \mathbb{E}_{\boldsymbol{i}^*}\left[N_{i_j^*}^{(j)}\right] \right)\notag\\
    &=\Delta \left(T - \sum_{j=1}^\Upsilon \frac{1}{K^{\Upsilon-1}}\sum_{\boldsymbol{i}_{-j}^* \in [K]^{\Upsilon-1}}\frac{1}{K}\sum_{i=1}^K\mathbb{E}_{(\boldsymbol{i}_{-j}^*,i)}\left[N_{i}^{(j)}\right] \right), \label{eq:lb_00}
    \end{align}
    where $\boldsymbol{i}_{-j}^*$ equals to $\boldsymbol{i}^*$ where the $j$-th coordinate is set to $0$ and $(\boldsymbol{i}_{-j}^*,i)$ equals to $\boldsymbol{i}^*$ where the $j$-th coordinate is set to $i$, for $i \in [K]$.

    Let $D_{KL}(P,Q)$ be the Kullback-Leibler divergence between $P$ and $Q$, then we have:
    \begin{align*}
        D_{KL}\left(\rho_{2^\frac{\epsilon}{1+\epsilon}\Delta},\rho_{\Delta}\right) &= \left(1-2v^{-\frac{1}{\epsilon}}\Delta^\frac{1+\epsilon}{\epsilon}\right)\log\left(\frac{1-2v^{-\frac{1}{\epsilon}}\Delta^\frac{1+\epsilon}{\epsilon}}{1-v^{-\frac{1}{\epsilon}}\Delta^\frac{1+\epsilon}{\epsilon}}\right) + 2v^{-\frac{1}{\epsilon}}\Delta^\frac{1+\epsilon}{\epsilon}\log\left(\frac{2v^{-\frac{1}{\epsilon}}\Delta^\frac{1+\epsilon}{\epsilon}}{v^{-\frac{1}{\epsilon}}\Delta^\frac{1+\epsilon}{\epsilon}}\right) \\
        &\le 2v^{-\frac{1}{\epsilon}}\Delta^\frac{1+\epsilon}{\epsilon}\log\left(2\right),
    \end{align*}
    where the inequality follows by upper bounding the first addendum with $0$.
    Using Pinsker Inequality and the previous bound on the KL divergence, for every $j \in [\Upsilon]$, we get
    \begin{align*}
        2\left(\frac{1}{K}\sum_{i=1}^K\mathbb{E}_{(\boldsymbol{i}_{-j}^*,i)}\left[\frac{N_{i}^{(j)}}{|E_j|}\right]-\frac{1}{K}\sum_{i=1}^K\mathbb{E}_{(\boldsymbol{i}_{-j}^*,0)}\left[\frac{N_{i}^{(j)}}{|E_j|}\right]\right)^2 &\le D_{KL}\left(\mathbb{P}_{(\boldsymbol{i}_{-j}^*,i)},\mathbb{P}_{(\boldsymbol{i}_{-j}^*,0)}\right) \\
        &\le \frac{1}{K}\sum_{i=1}^K \mathbb{E}_{(\boldsymbol{i}_{-j}^*,i)}\left[N_{i}^{(j)}\right]D_{KL}\left(\rho_{2^\frac{\epsilon}{1+\epsilon}\Delta},\rho_{\Delta}\right)\\
        &\le \frac{\log(2)}{2K}|E_j| v^{-\frac{1}{\epsilon}}\Delta^\frac{1+\epsilon}{\epsilon},
    \end{align*}
    that implies
    \begin{align}
    \label{eq:lb_01}
        \frac{1}{K}\sum_{i=1}^K\mathbb{E}_{(\boldsymbol{i}_{-j}^*,i)}\left[N_{i}^{(j)}\right] \le \frac{|E_j|}{K} + \sqrt{\frac{\log(2)}{2K}}|E_j|^\frac{3}{2} v^{-\frac{1}{2\epsilon}}\Delta^\frac{1+\epsilon}{2\epsilon}.
    \end{align}
    Combining Equation \eqref{eq:lb_00} and Equation \eqref{eq:lb_01}, we get
    \begin{align*}
         \frac{1}{K^\Upsilon}\sum_{\boldsymbol{i}^* \in [K]^\Upsilon}\mathbb{E}_{\boldsymbol{i}^*}\left[R^\pi(T)\right]&\ge \left(\frac{T}{2} - \sum_{k=1}^\Upsilon \sqrt{\frac{\log(2)}{2K}}|E_j|^\frac{3}{2}v^{-\frac{1}{2\epsilon}} \Delta^\frac{1+\epsilon}{2\epsilon}\right)\Delta \\
         &\ge \frac{1}{2}\left(\frac{2\log(2)}{16}\right)^\frac{\epsilon}{1+\epsilon}K^\frac{\epsilon}{1+\epsilon}\Upsilon^\frac{\epsilon}{1+\epsilon}T^\frac{1}{1+\epsilon},
    \end{align*}
    by setting $\Delta = v^{\frac{1}{1+\epsilon}}\left(\frac{2\log(2)K\Upsilon}{16T}\right)^\frac{\epsilon}{1+\epsilon}$.
\end{proof}
\begin{remark}
\label{rem:lb_easier}
    In the proof of Theorem \ref{thr:lb}, we do not impose any condition on $T$. Moreover, the length of every epoch is equal to $T/\Upsilon$. This means that in principle one can choose a large enough $T$, \emph{i.e.}, $T\ge 2\Upsilon\lceil L_j \frac{K}{\eta}\rceil$, such that Assumption \ref{ass:length} holds. This proves that our assumption does not make the problem easier from a regret minimization perspective.
\end{remark}
\detectionDelay*
\begin{proof}
Due to its length, we divided this proof into several steps. In Steps 1-3 we extend Theorem 10 of \cite{wang2023catoni} to the case of heavy-tailed random variables. In Step 4 we apply the \textit{sticthing} technique to the resulting CS, tightening its width. Then, in Step 5 we define the CS hyper-parameters and define a set of good events, under which we are able to properly bound the detection delay in Step 7. The proof is concluded by showing that no false alarm occurs under the good event (Step 8).\\
    \textbf{Step 1 (Building a nonnegative supermartingale)} First, we observe that
    \begin{align*}
        M_t &\coloneqq \prod_{i=1}^t \exp\left\{\phi_\epsilon(\lambda_i(X_i -\mu))-\lambda_i^{1+\epsilon}\frac{v}{1+\epsilon}\right\}, \\
        N_t &\coloneqq \prod_{i=1}^t \exp\left\{-\phi_\epsilon(\lambda_i(X_i -\mu))-\lambda_i^{1+\epsilon}\frac{v}{1+\epsilon}\right\},
    \end{align*}
    are nonnegative supermartingales. To prove this for $M_t$ (all steps are analogous for $N_t$), we bound
    \begin{align*}
        \mathbb{E}&\left[\exp\left\{\phi_\epsilon(\lambda_t(X_t-\mu))-\lambda_t\frac{v}{1+\epsilon} \right\}\,\,\bigg\lvert\, \mathcal{F}_{t-1}\right] \le\\
        &\le \mathbb{E}\left[1+\lambda_t(X_t-\mu)+\lambda_t^{1+\epsilon}\frac{(X_t-\mu)^{1+\epsilon}}{1+\epsilon} \,\,\bigg\lvert\, \mathcal{F}_{t-1}\right]\exp\left\{-\lambda_t^{1+\epsilon}\frac{v}{1+\epsilon}\right\}\\
        &\le \left(1+\lambda_t^{1+\epsilon}\frac{v}{1+\epsilon}\right)\exp\left\{-\lambda_t^{1+\epsilon}\frac{v}{1+\epsilon}\right\} \le 1,
    \end{align*}
    and, subsequently,
    \begin{align*}
        \mathbb{E}\left[M_t \,\,\bigg\lvert\, \mathcal{F}_{t-1} \right] = M_{t-1}\mathbb{E}\left[ \exp\left\{\phi_\epsilon(\lambda_t(X_t -\mu))-\lambda_t^{1+\epsilon}\frac{v}{1+\epsilon}\right\} \,\,\bigg\lvert\, \mathcal{F}_{t-1} \right] \le M_{t-1}.
    \end{align*}
    \textbf{Step 2 (Building a CS for $\phi_{\epsilon}$)} Then, we can leverage Ville's inequality to construct a CS around $\phi_{\epsilon}(\lambda(X-\mu))$:
    \begin{align*}
        \mathbb{P}\left(\exists t \ge 1 : M_t \ge \frac{2}{\gamma}\right) \le \frac{\gamma}{2},
    \end{align*}
    which implies
    \begin{align*}
        \mathbb{P}\left(\exists t \ge 1 : \sum_{i=1}^t \phi_{\epsilon}(\lambda_i(X_i-\mu)) \ge \frac{v \sum_{i=1}^t\lambda_i^{1+\epsilon}}{1+\epsilon}+\log \left(\frac{2}{\gamma}\right)\right) \le \frac{\gamma}{2}.
    \end{align*}
    Analogous calculations for $N_t$ and a union bound, yield a $(1-\gamma)$-CS where the intervals have the following form:
    \begin{equation*}
        CI_t^\phi = \left\{ m \in \mathbb{R} : -\frac{v \sum_{i=1}^t\lambda_i^{1+\epsilon}}{1+\epsilon}-\log \left(\frac{2}{\gamma}\right) \le \sum_{i=1}^t \phi_{\epsilon}(\lambda_i(X_i-m)) \le \frac{v \sum_{i=1}^t\lambda_i^{1+\epsilon}}{1+\epsilon}+\log \left(\frac{2}{\gamma}\right)\right\}.
    \end{equation*}
    \textbf{Step 3 (Bounding the width of the CS for $\mu$)} We are now required to provide a bound on the width of the previously derived $(1-\gamma)$-CS. To do so, we derive high-probability lower and upper bounds over the random solution of $f_t(m)\coloneqq \sum_{i=1}^t \phi_{\epsilon}(\lambda_i(X_i-m))=0$.
    For all $m \in \mathbb{R}$, let
    \begin{align*}
        M_t(m) = \exp \left\{f_t(m)-\sum_{i=1}^t \left(\lambda_i(\mu-m) + \frac{\lambda_i^{1+\epsilon}}{1+\epsilon}(v + (\mu-m)^{1+\epsilon})\right)\right\},
    \end{align*}
    then, with steps analogous to Step 1, we observe that $M_t(m)$ is a nonnegative supermartingale. Note that $M_t(\mu)=M_t$, and an analogous definition leads to the nonnegative supermartingale $N_t(m)$. We define:
    \begin{align*}
        B_t^+(m) &= \sum_{i=1}^t \left(\lambda_i(\mu-m) +\frac{\lambda_i^{1+\epsilon}}{1+\epsilon}(v + (\mu-m)^{1+\epsilon})\right) + \log\left(\frac{2}{\theta}\right) \\
        B_t^-(m) &= \sum_{i=1}^t \left(\lambda_i(\mu-m) - \frac{\lambda_i^{1+\epsilon}}{1+\epsilon}(v + (\mu-m)^{1+\epsilon})\right) - \log\left(\frac{2}{\theta}\right), 
    \end{align*}
    and using Markov's ineqality we get:
    \begin{align*}
        \forall m \in \mathbb{R},~~~ \mathbb{P}\left(f_t(m) \le B_t^+(m)\right) \ge 1-\frac{\theta}{2} \\
        \forall m \in \mathbb{R},~~~ \mathbb{P}\left(f_t(m) \ge B_t^-(m)\right) \ge 1-\frac{\theta}{2}.
    \end{align*}
    Since $B_t^+$ upper bounds $f_t(m)$ with probability at least $1-\frac{\theta}{2}$, any $\widetilde{m}_t$ s.t.
    \begin{align}
    \label{eq:catoni_cs_0}
        B_t^+(\widetilde{m}_t) = - \sum_{i=1}^t \frac{v\lambda_i^{1+\epsilon}}{1+\epsilon}- \log \left(\frac{2}{\gamma}\right) = f_t(\max\{CI_t^\phi\})
    \end{align}
    also satisfies
    \begin{align}
    \label{eq:catoni_cs_1}
        \mathbb{P}\left(f_t(\widetilde{m}_t) \le f_t(\max\{CI_t^\phi\}\right) \ge 1-\frac{\theta}{2},
    \end{align}
    where $\widetilde{m}_t$ is a non-random quantity as it's the solution to a deterministic equation. 
    As $f_t(m)$ is a non-increasing function of $m$, Equation \eqref{eq:catoni_cs_1} implies that:
    \begin{align*}
        \mathbb{P}\left(\widetilde{m}_t \le\max\{CI_t^\phi\}\right) \ge 1-\frac{\theta}{2}.
    \end{align*}
    Note that Equation \eqref{eq:catoni_cs_0} admits solutions if and only if
    \begin{equation}
    \label{eq:catoni_cs_condition}
        \left(\sum_{i=1}^t \lambda_i^{1+\epsilon}\right)^\frac{1}{\epsilon}\left(\sum_{i=1}^t \lambda_i\right)^{-\frac{1+\epsilon}{\epsilon}}\left(\sum_{i=1}^t 5 \lambda_i^{1+\epsilon} \frac{v}{1+\epsilon} + 2 \log \left( \frac{2}{\gamma}\right) + 2 \log\left( \frac{2}{\theta}\right)\right) \le \frac{\epsilon}{1+\epsilon}.
    \end{equation}
    Finally, we conclude this step by bounding
    \begin{align*}
        \widetilde{m}_t \le \mu + \frac{\sum_{i=1}^t 10 v\lambda_i^{1+\epsilon}+2(1+\epsilon)\log\left(\frac{2}{\gamma}\right)+2(1+\epsilon)\log\left(\frac{2}{\theta}\right)}{\sum_{i=1}^t \lambda_i},
    \end{align*}
    which yields the upper CS on $\mu$ in the following form:
    \begin{align*}
        \mathbb{P}\left(\max\{CI_t^\phi\} \le \mu +  \frac{\sum_{i=1}^t 10 v\lambda_i^{1+\epsilon}+2(1+\epsilon)\log\left(\frac{2}{\gamma}\right)+2(1+\epsilon)\log\left(\frac{2}{\theta}\right)}{\sum_{i=1}^t \lambda_i}\right) \ge 1-\frac{\theta}{2}.
    \end{align*}
    Repeating all the previous steps for $B_t^-(m)$, and by applying a union bound, yields a two-sided $(1-\gamma)$-CS for $\mu$. The width $w(t, \mu, \gamma) = \max\{CI_t^\phi\}-\min\{CI_t^\phi\}$ of such CS concentrates as 
    \begin{align}
    \label{eq:catoni_cs_2}
        \mathbb{P}\left(w(t,\mu,\gamma) \le 2\frac{\sum_{i=1}^t 10 v\lambda_i^{1+\epsilon}+2(1+\epsilon)\log\left(\frac{2}{\gamma}\right)+2(1+\epsilon)\log\left(\frac{2}{\theta}\right)}{\sum_{i=1}^t \lambda_i}\right) \ge 1-\theta.
    \end{align}
\textbf{Step 4 (Stitching)} We now discuss the choice of the sequence $\{\lambda_t\}_{t\ge 1}$. The idea is to partition time in an exponential grid, and then fix the same value of $\lambda_t$ inside the same cell. Moreover, the confidence level is modified and set to a cell-specific value $\gamma_j$. This idea, called \textit{stitching}, first appeared in \cite{howard2021time}. In particular, set $t_j = e^j$, $\gamma_j = \frac{\gamma}{(j+1)^2}$, and $\Lambda_j = \left(\log \left(\frac{2}{\gamma_j}\right)e^{-j}v^{-1}\right)^{\frac{1}{1+\epsilon}}$. Then, for every $t_j < t \le t_{j+1}$, we set $\lambda_i = \Lambda_j$ for every $i \in [t]$.
Assume $\theta = \frac{\gamma}{4}$. For every $t_j < t \le t_{j+1}$, we have
\begin{align*}
    &\frac{\sum_{i=1}^t 10 v\lambda_i^{1+\epsilon}+2(1+\epsilon)\log\left(\frac{2}{\gamma_j}\right)+2(1+\epsilon)\log\left(\frac{2}{\theta}\right)}{\sum_{i=1}^t \lambda_i}  \\
    &= v^{\frac{1}{1+\epsilon}}\frac{10t v\Lambda_j^{1+\epsilon}+2(1+\epsilon)\log\left(\frac{2}{\gamma_j}\right)+2(1+\epsilon)\log\left(\frac{2}{\theta}\right)}{t\Lambda_j} \le \\
    &\le v^{\frac{1}{1+\epsilon}}(1+\epsilon)\frac{\frac{10v}{1+\epsilon}t \Lambda_j^{1+\epsilon}+4\log\left(\frac{2}{\gamma_j}\right)}{t\Lambda_j} \le \\
    &\le 34v^{\frac{1}{1+\epsilon}}(1+\epsilon)\left(\frac{\log \left(\frac{2}{\gamma}\right) + 2\log(\log(e^2t))}{t}\right)^{\frac{\epsilon}{1+\epsilon}}.
\end{align*}
Noting that $\sum_{j=1}^\infty \gamma_j < \gamma$, this yields a tight bound over the width of the $(1-\gamma)$-CS for $\mu$.

\textbf{Step 5 (Good event characterization)} As the width derived in the previous steps is not deterministic, we now characterize a favorable event in which such bound hold simultaneously for all $(1-\gamma)$-CS. For any $(1-\gamma)$-CS of length $t$, we have that 
\begin{align*}
     \mathbb{P}\left(w(t,\mu,\gamma) \le 68v^{\frac{1}{1+\epsilon}}(1+\epsilon)\left(\frac{\log \left(\frac{2}{\gamma}\right) + 2\log(\log(e^2t))}{t}\right)^{\frac{\epsilon}{1+\epsilon}}\right) \ge 1-\frac{\gamma}{4}.
\end{align*}
Thus, considering a stream of $T$ samples, the probability of this to be violated for at least one interval of the $(1-\gamma)$-CS is bounded as
\begin{align*}
   1-\mathbb{P}(\mathcal{W}_T)&=\mathbb{P}\left(\exists t \le T : w(t,\mu,\gamma) > 68v^{\frac{1}{1+\epsilon}}(1+\epsilon)\left(\frac{\log \left(\frac{2}{\gamma}\right) + 2\log(\log(e^2t))}{t}\right)^{\frac{\epsilon}{1+\epsilon}}\right) \\
     &\le\sum_{t=1}^T \mathbb{P}\left( w(t,\mu,\gamma) > 68v^{\frac{1}{1+\epsilon}}(1+\epsilon)\left(\frac{\log \left(\frac{2}{\gamma}\right) + 2\log(\log(e^2t))}{t}\right)^{\frac{\epsilon}{1+\epsilon}}\right) \\
     &\le T\frac{\gamma}{4}.
\end{align*}
The event $\mathcal{W}_T$, defined above, represents a good event in which the $(1-\gamma)$-CS starting from $t=1$ have the widths of the single CIs deterministically bounded up until horizon $T$. Now, we note that if we have $t$ different $(1-\gamma)$-CS of lengths $1,\ldots,t$, we define $ \mathcal{W}_{1:t} \coloneqq \bigcap_{i=1}^t \mathcal{W}_i$. This event describe the scenario in which all $(1-\gamma)$-CS starting sequentially before $t$ have the widths of all of their CIs bounded. Using another union bound argument, we can see that $\mathbb{P}(\mathcal{W}_{1:t}) \ge 1-\frac{t(t+1)\gamma}{8}$. Finally, note that $\mathcal{W}_{a:b} \subset \mathcal{W}_{a':b'}$, for every $a'>a$ and $b'<b$. Thus $\mathbb{P}(\mathcal{W}_{a:b}) \ge 1-\frac{T(T+1)\gamma}{8}$ for every $a,b \in [T]$. Characterizing this event is necessary since \texttt{Catoni-FCS-detector} requires that CS widths \textit{well behave}, \emph{i.e.}, they possess a deterministic upper bound.

We also introduce the event $\mathcal{E}_t^T = \left\{\forall i \in \{t,\ldots,T\}, \forall t' \in \{i,\ldots,T\} : \mu \in CI_{t'}^{(i)}\right\}$, that represents the scenario in which every $(1-\gamma)$-CS starting from a timestamp greater or equal than $t$ never miscovers the true mean up to time $T$. By the definition of $(1-\gamma)$-CS, we have $\mathbb{P}(\mathcal{E}_t^T) \ge 1-(T-t)\gamma$.
From now on, we continue by setting $\gamma = \frac{2}{T^3}$.

\paragraph{Step 6 (Verifying condition \eqref{eq:catoni_cs_condition})}
For every $t\in [T]$, we use the previously defined values for $\{\lambda_i\}_{i=1}^t$, $\gamma$ and $\theta$, and solve inequality \eqref{eq:catoni_cs_condition}. We obtain that it is satisfied for every $t \ge 68\log\left(T^\frac{1+\epsilon}{\epsilon}\right) = n_{min}$, which is always true under the theorem's assumptions.

\textbf{Step 7 (Bounding the detection delay)} We are now ready to bound the detection delay of \texttt{Catoni-FCS-detector} after a change of magnitude $\delta$ happened after $t_c$ samples. Note that we assume $t_c$ to be large enough to satisfy Equation \eqref{eq:catoni_cs_condition}. To do so, we leverage the width of the $(1-\gamma)$-CS that has just been derived. Suppose, without loss of generality, that a change point is detected after at most $T$ overall samples. Thus, we work under the events $\mathcal{W}_{1:T}$, $\mathcal{E}_1^{t_c}$, and $\mathcal{E}_{t_c}^T$, defined in Step 5, which hold simultaneously with probability at least $1-\frac{14}{T}$, and guarantee that the CS widths are always properly bounded and the pre-change mean $\mu_0$ and the post-change mean $\mu_1$ are never miscovered. By assumption, $t_c$ is large enough to ensure
\begin{align*}
    w(t_c, \mu_0, \gamma) \le 68v^{\frac{1}{1+\epsilon}}(1+\epsilon)\left(\frac{3\log \left(T\right) + 2\log(\log(e^2t_c))}{t_c}\right)^{\frac{\epsilon}{1+\epsilon}} \le \frac{\delta}{2}
\end{align*}
and we have to find an $n$ s.t.:
\begin{align*}
    w(n, \mu_1, \gamma) \le 68v^{\frac{1}{1+\epsilon}}(1+\epsilon)\left(\frac{3\log \left(T\right) + 2\log(\log(e^2n))}{n}\right)^{\frac{\epsilon}{1+\epsilon}} \le \frac{\delta}{2}.
\end{align*}
We first bound $2\log(\log(e^2n))\le 3\log(\log(n))$, that holds under the trivial requirements that $\log(n) \ge 2$ and $T\ge 2$. Moreover, we define $\widetilde{c} \coloneqq 136(3)^\frac{\epsilon}{1+\epsilon}(v)^{\frac{1}{1+\epsilon}}$ and $\widetilde{\delta} = \frac{\delta}{\widetilde{c}}$. Thus, we can find an upper bound on the expected detection delay $n_0$ by solving the following:
\begin{align*}
    n_0 = \min_{n \ge 1}\left\{ \left(\frac{\log \left(T\right) + \log(\log(n))}{n}\right)^{\frac{\epsilon}{1+\epsilon}} \le \widetilde{\delta}\right\}.
\end{align*}
If $\widetilde{\delta} \ge 1$, then $n_0 \le \log(T)$. Else, for $\widetilde{\delta} \le 1$, we define 
\begin{align*}
    n_1 = \min_{n \ge 1}\left\{\left(\frac{ \log(\log(n))}{n}\right)^{\frac{\epsilon}{1+\epsilon}} \le \frac{\widetilde{\delta}}{2}\right\} ~~~\text{and}~~~ n_2 = \min_{n \ge 1}\left\{\left(\frac{ \log(T)}{n}\right)^{\frac{\epsilon}{1+\epsilon}} \le \frac{\widetilde{\delta}}{2}\right\},
\end{align*}
and note that $n_0 \le n_1 + n_2$. We can thus upper bound them separately. It is trivial to observe that
\begin{align*}
    n_2 = 2^{\frac{1+\epsilon}{\epsilon}}\frac{\log(T)}{\widetilde{\delta}^\frac{1+\epsilon}{\epsilon}} = (2\widetilde{c})^\frac{1+\epsilon}{\epsilon}\frac{\log(T)}{\delta^\frac{1+\epsilon}{\epsilon}}.
\end{align*}
Upper bounding $n_1$ requires additional effort. We start by identifying a value $n_3$ which satisfies
\begin{align*}
    \left(\frac{ \log(\log(n_3))}{n_3}\right)^{\frac{\epsilon}{1+\epsilon}} \le \frac{\widetilde{\delta}}{2}.
\end{align*}
Let $n_3 = \left(\frac{4}{\widetilde{\delta}^2}\right)^{\frac{1+\epsilon}{\epsilon}}$, and $\widetilde{y}\coloneqq \left(\frac{\widetilde{\delta}}{2}\right)^{\frac{1+\epsilon}{\epsilon}} \le 1$, then
\begin{align*}
    \left(\frac{2}{\widetilde{\delta}}\left(\frac{ \log(\log(n_3))}{n_3}\right)^{\frac{\epsilon}{1+\epsilon}}\right)^{\frac{1+\epsilon}{\epsilon}} &= \left(\frac{2}{\widetilde{\delta}}\right)^{\frac{1+\epsilon}{\epsilon}}\left(\frac{ \log(\log(n_3))}{n_3}\right) \\
    &= \left(\frac{\widetilde{\delta}}{2}\right)^{\frac{1+\epsilon}{\epsilon}}\log\left(\log\left(\left(\frac{4}{\widetilde{\delta}^2}\right)^{\frac{1+\epsilon}{\epsilon}}\right)\right)\\
    &= \left(\frac{\widetilde{\delta}}{2}\right)^{\frac{1+\epsilon}{\epsilon}}\log\left(\log\left(\left(\frac{2}{\widetilde{\delta}}\right)^{2\frac{1+\epsilon}{\epsilon}}\right)\right)\\
    &= \widetilde{y}\log\left(\log\left(\frac{1}{\widetilde{y}^2}\right)\right) \le 0.27 < 1.
\end{align*}
Since $n_3$ is an upper bound on the expected detection delay, we have 
\begin{align*}
    \log(\log(n_1)) &\le \log(\log(n_3)) \\
    &= \log\left(\log\left(\left(\frac{4}{\widetilde{\delta}^2}\right)^{\frac{1+\epsilon}{\epsilon}}\right)\right) \\
    &= \log\left(2\frac{1+\epsilon}{\epsilon}\log\left(\frac{2}{\widetilde{\delta}}\right)\right) \\
    &= \log\left(2\frac{1+\epsilon}{\epsilon}\right) + \log\left(\log\left(\frac{2\widetilde{c}}{\delta}\right)\right) \\
    &\le \log\left(2\frac{1+\epsilon}{\epsilon}\right) + \log(\log(2\widetilde{c}))+\log\left(\log\left(\frac{1}{{\delta}}\right)\right)\\
    &= \log \left(2\log(2\widetilde{c})\frac{1+\epsilon}{\epsilon}\right)+\log\left(\log\left(\frac{1}{{\delta}}\right)\right).
\end{align*}
As a consequence, we can rewrite:
\begin{align*}
    \left(\frac{ \log(\log(n_1))}{n_1}\right)^{\frac{\epsilon}{1+\epsilon}} \le \left(\frac{ \log \left(2\log(2\widetilde{c})\frac{1+\epsilon}{\epsilon}\right)+\log\left(\log\left(\frac{1}{{\delta}}\right)\right)}{n_1}\right)^{\frac{\epsilon}{1+\epsilon}},
\end{align*}
which immediately implies that 
\begin{align*}
    n_1 &\le \frac{(2\widetilde{c})^\frac{1+\epsilon}{\epsilon}\log \left(2\log(2\widetilde{c})\frac{1+\epsilon}{\epsilon}\right)+(2\widetilde{c})^\frac{1+\epsilon}{\epsilon}\log\left(\log\left(\frac{1}{{\delta}}\right)\right)}{{\delta}^{\frac{1+\epsilon}{\epsilon}}}.
\end{align*}

Under the events defined above and that hold with probability at least $1-\frac{14}{T}$, the detection delay is bounded as
\begin{align*}
   (\tau-t_c)^+ &\le (2\widetilde{c})^\frac{1+\epsilon}{\epsilon}\frac{\log \left(2\log(2\widetilde{c})\frac{1+\epsilon}{\epsilon}\right)+\log\left(\log\left(\frac{1}{{\delta}}\right)\right)+\log(T)}{{\delta}^{\frac{1+\epsilon}{\epsilon}}} \\
   &\le 6 (472)^\frac{1+\epsilon}{\epsilon}v^\frac{1}{\epsilon}\frac{\log\left(\log\left(\frac{1}{{\delta}}\right)\right)+\log(T)}{{\delta}^{\frac{1+\epsilon}{\epsilon}}}.
\end{align*}
\paragraph{Step 8 (Bounding the probability of false alarm)}
The bound on the probability of false alarm is a trivial consequence of the definition of event $\mathcal{E}_1^{t_c}$. Under this event, it is impossible by construction for the detector to raise a false alarm, as all the CS always intersect at least on $\mu_0$. Thus, the probability of false alarm is bounded by $\mathbb{P}\left((\mathcal{E}_1^{t_c})^C\right)\le \frac{1}{T}$.
\end{proof}
\begin{lemma}
\label{lemma:best_event}
    Let $\mathcal{G}_T \coloneqq \left\{ \forall j\in[\Upsilon] : \tau_j \in \left\{t_c^{(j)},\ldots, t_c^{(j)} + \left\lceil L_j \frac{K}{\eta}\right\rceil\right\} \text{ and } t_c^{(\Upsilon+1)}> T\right\}$ be the event in which \algnameshort restarts exactly $\Upsilon$ times without false alarms and excessive delays.
    Then, we have $\mathbb{P}\left(\mathcal{G}_T^C\right) \le \frac{15K\Upsilon}{T}$.
\end{lemma}
\begin{proof}
   Note that, by construction of the algorithm, each action is sampled at least $L_j$ times after $\left\lceil L_j \frac{K}{\eta}\right\rceil$ timesteps have passed since the last detection point. Thanks to Assumption \ref{ass:length}, the length of every epoch is at least $2\left\lceil L_j \frac{K}{\eta}\right\rceil$. 

   Let $\mathcal{L}^{(j)}\coloneqq \left\{\forall m \le j : \tau_m \in \left\{t_c^{(m)}, \ldots, t_c^{(m)}+\left\lceil L_m \frac{K}{\eta}\right\rceil\right\}\right\}$ be the event in which all detections happened without false alarms and excessive delays up to the $j$-th epoch. Then, by a union bound and by Proposition \ref{prop:detection_delay}, we have:
   \begin{align*}
       \mathbb{P}\left(\mathcal{G}_T^C\right) &\le  \sum_{j=1}^{\Upsilon+1}\mathbb{P}\left(\tau_j \le t_c^{(j)} \mid \mathcal{L}^{(j-1)}\right) + \sum_{j=1}^{\Upsilon}\mathbb{P}\left(\tau_j \ge t_c^{(j)} + \left\lceil L_j \frac{K}{\eta}\right\rceil \mid \mathcal{L}^{(j-1)}\right) \\
       &\le \frac{K(\Upsilon+1)}{T} + \frac{14K\Upsilon}{T} \le \frac{15K\Upsilon}{T}.
   \end{align*}
\end{proof}

\regretBound*
\begin{proof}
Let $\mathcal{G}_T \coloneqq \left\{ \forall j\in[\Upsilon] : \tau_j \in \left\{t_c^{(j)},\ldots, t_c^{(j)} + \left\lceil L_j \frac{K}{\eta}\right\rceil\right\} \text{ and } t_c^{(\Upsilon+1)}> T\right\}$ be the event in which \algnameshort restarts exactly $\Upsilon$ times without false alarms and excessive delays.

We start by decomposing the regret in the following way:
\begin{align*}
        \mathbb{E}[R^{\pi^\text{\algnameshort}}(T)] &\le  \mathbb{E}[R^{\pi^\text{\algnameshort}}(T) \mid \mathcal{G}_T] + \mathbb{E}[R^{\pi^\text{\algnameshort}}(T) \mid \mathcal{G}_T^C] \mathbb{P}\left(\mathcal{G}_T^C\right)  \\
        &\le \mathbb{E}[R^{\pi^\text{\algnameshort}}(T) \mid \mathcal{G}_T] + 15K\Upsilon,
\end{align*}
where the second inequality follows from Lemma \ref{lemma:best_event}. We can now focus on bounding the first addendum. We decompose it as follows:
\begin{align*}
    \mathbb{E}[R^{\pi^\text{\algnameshort}}(T) \mid \mathcal{G}_T] &= \mathbb{E}[R^{\pi^\text{\algnameshort}}(T) - R^{\pi^\text{\algnameshort}}(t_c^{(1)}) \mid \mathcal{G}_T] + \mathbb{E}[R^{\pi^\text{\algnameshort}}(t_c^{(1)}) \mid \mathcal{G}_T] \\
    &\le \mathbb{E}[R^{\pi^\text{\algnameshort}}(T) - R^{\pi^\text{\algnameshort}}(t_c^{(1)}) \mid \mathcal{G}_T] + \eta t_c^{(1)}\Delta_{max}^{(1)} + \mathbb{E}[R^{\pi_s}(t_c^{(1)})],
\end{align*}
where the inequality follows by upper bounding the contribution to the regret given by the forced exploration in the first $t_c^{(1)}$ rounds, the remaining term is the expected regret accrued by the policy $\pi_s$ up to $t_c^{(1)}$.
We prosecute by bounding the first addendum as follows:
\begin{align*}
    \mathbb{E}[R^{\pi^\text{\algnameshort}}(T) - R^{\pi^\text{\algnameshort}}(t_c^{(1)}) \mid \mathcal{G}_T] &= \mathbb{E}[R^{\pi^\text{\algnameshort}}(T) - R^{\pi^\text{\algnameshort}}(\tau_1) \mid \mathcal{G}_T]  + \\
    &~~~~~+\mathbb{E}[R^{\pi^\text{\algnameshort}}(\tau_1) - R^{\pi^\text{\algnameshort}}(t_c^{(1)}) \mid \mathcal{G}_T] \\
    &\le \mathbb{E}_2[R^{\pi^\text{\algnameshort}}(T-\tau_1) \mid \mathcal{G}_T]  + \mathbb{E}[(\tau_1-t_c^{(1)})\mid \mathcal{G}_T]\Delta_{max}^{(1)},
\end{align*}
where $\mathbb{E}_2$ is the expectation according to an environment starting from the second segment.
Putting all together, we can write
\begin{equation*}
\mathbb{E}[R^{\pi^\text{\algnameshort}}(T) \mid \mathcal{G}_T] \le \mathbb{E}_2[R^{\pi^\text{\algnameshort}}(T-\tau_1) \mid \mathcal{G}_T]  + \mathbb{E}[(\tau_1-t_c^{(1)})\mid \mathcal{G}_T]\Delta_{max}^{(1)} + \eta t_c^{(1)}\Delta_{max}^{(1)} + \mathbb{E}[R^{\pi_s}(t_c^{(1)})],
\end{equation*}
which yields, by a recursive application:
\begin{align*}
    \mathbb{E}[R^{\pi^\text{\algnameshort}}(T) \mid \mathcal{G}_T] &\le \sum_{j=1}^\Upsilon \mathbb{E}[(\tau_j-t_c^{(j)})\mid \mathcal{G}_T]\Delta_{max}^{(1)} + \sum_{j=1}^\Upsilon \mathbb{E}[R^{\pi_s}(|E_j|)]  + \eta T\Delta_{max}^{(1)} \\
    &\le \sum_{j=1}^\Upsilon \left\lceil L_j \frac{K}{\eta}\right\rceil\Delta_{max}^{(1)} + \sum_{j=1}^\Upsilon \mathbb{E}[R^{\pi_s}(|E_j|)]  + \eta T\Delta_{max}^{(1)},
\end{align*}
where the second inequality follows from the definition of $\mathcal{G}_T$. The proof is concluded by substituting $L_j$ with its definition.
\end{proof}
\regretBoundInstance*
\begin{proof}
    The proof of this theorem trivially follows from plugging the regret bounds of \texttt{Robust UCB} with MoM estimator (Theorem 3 and Proposition 1 of \cite{bubeck2013bandits}). Equation \eqref{eq:regret_instance_indep} necessitates an additional step using Jensen Inequality:
    $$
    \sum_{j=1}^\Upsilon K^\frac{\epsilon}{1+\epsilon}(vT)^\frac{1}{1+\epsilon} \le \Upsilon (K)^\frac{\epsilon}{1+\epsilon}\left(\frac{vT}{\Upsilon}\right)^\frac{1}{1+\epsilon} =  (\Upsilon K)^\frac{\epsilon}{1+\epsilon}(vT)^\frac{1}{1+\epsilon}.
    $$
\end{proof}
\regretBoundFinal*
\begin{proof}
    First, note that the proof of Theorem \ref{thr:ub} can be conducted in the exact same way by substituting $\eta$ with the sequence $\{\eta_j\}_{j \in [\Upsilon+1]}$. Note that, thanks to event $\mathcal{G}_T$ the algorithm restarts exactly $\Upsilon$ times.
    Equation \eqref{eq:final_regret} is a trivial consequence of the fact that $\eta_{\Upsilon+1} \ge \eta_j$ for every $j \le \Upsilon$, due to the monotonocity of the sequence. Moreover, we bound $\Delta_{max}^{(j)} \le \Delta_{max}$.

    To prove Equation \eqref{eq:final_regret_id}, we need an additional step. In particular:
    \begin{align*}
        \sum_{j=1}^\Upsilon \frac{1}{\eta_j} = \frac{1}{\eta_0}\sqrt{\frac{T}{K\log (T)}}\sum_{j=1}^\Upsilon \frac{1}{\sqrt{j}} \le \frac{1}{\eta_0}\sqrt{\frac{\Upsilon T}{K\log (T)}}.
    \end{align*}
    Plugging this in \textcolor{blue}{(A)}, and bounding $\widetilde{\delta}_{min}^{(j)} \ge \delta_{min}$ for every $j \in [\Upsilon]$, concludes the proof. 
\end{proof}
\section{Additional Numerical Evaluations}
\label{apx:exp}
In this appendix, we provide additional details on the experimental evaluation of Section \ref{sec:experiments} and additional experimental campaigns in synthetic environments.

\subsection{Detection Delay Analysis}
We evaluate how reactive is \texttt{Catoni-FCS-detector} to changes of data-generating distribution, and comparing it \texttt{repeated-FCS-detector} with Empirical Bernstein CSs from \cite{shekhar2023reducing}, Section 2.2, which is suited for distributions with finite variance.
We consider two distribution-shift scenarios: Gaussian distributions with $\sigma=1$ and Laplace distributions with scale equal to $1$. The change happens after $t_c = 400$ steps, and the total horizon is $T=1000$. The magnitude of change is $\delta = 1$. In Figure \ref{fig:results_delay}, we report the distribution of the detection delay of both algorithm over $20$ trials. We can see how, in general, \texttt{Catoni-FCS-detector} has a smaller detection delay w.r.t. \texttt{repeated-FCS-detector}. Moreover, no false alarm is raised along the $20$ trials.

\begin{figure}[tp!]
    \centering
    \subfloat[Gaussian samples.]{
        \resizebox{0.48\linewidth}{!}{\includegraphics{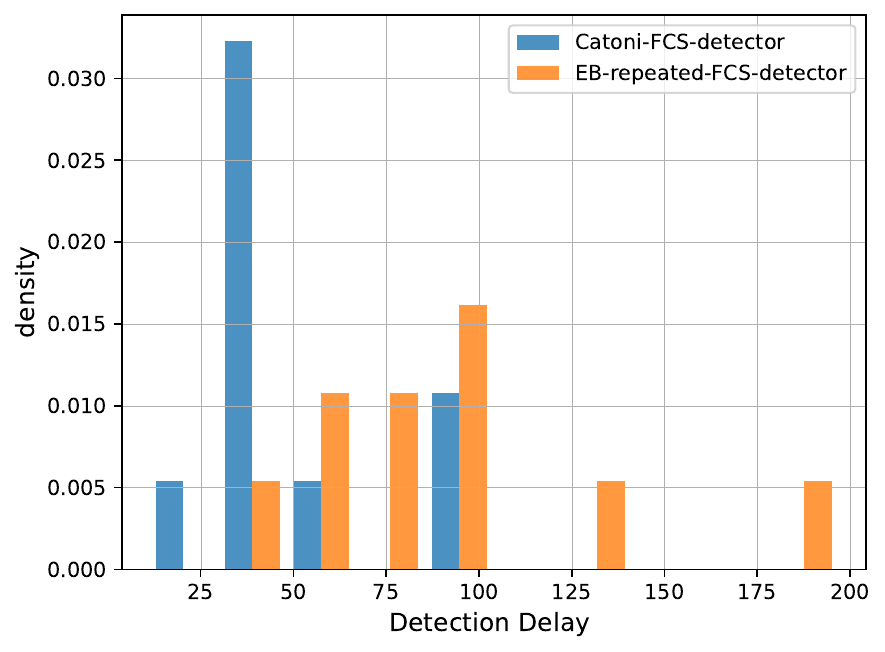}}
        \label{fig:gauss}
    }
    \subfloat[Laplace samples.]{
        \resizebox{0.48\linewidth}{!}{\includegraphics{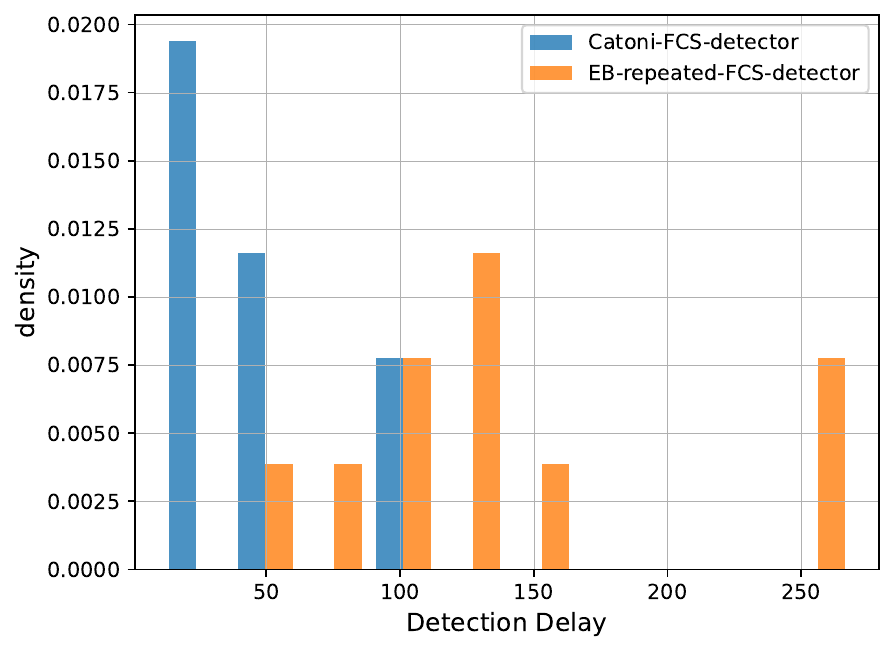}}
        \label{fig:laplace}
    }
    \caption{Distribution delay distribution over $20$ trials.}
    \label{fig:results_delay}
\end{figure}

\subsection{Regret Minimization in Highly Non-Stationary Environments}
\begin{figure}[tp!]
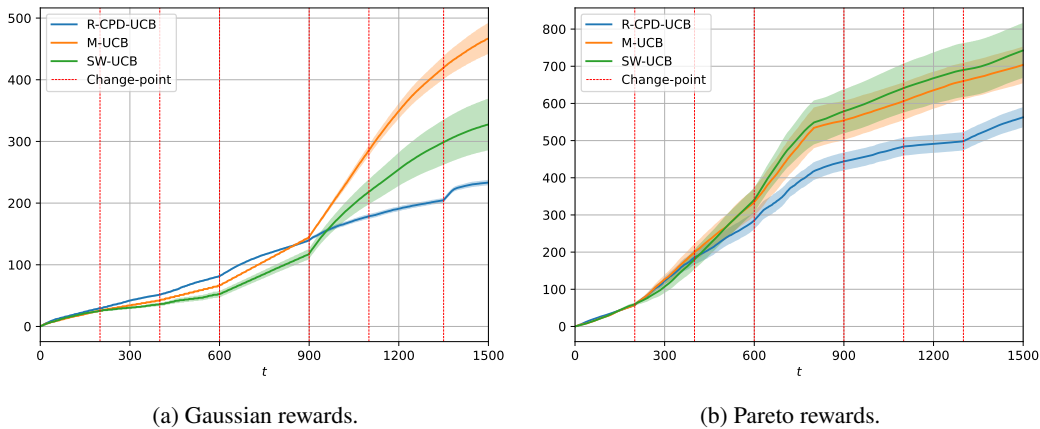

    \centering
    \subfloat[Gaussian rewards.]{
        \resizebox{0.49\linewidth}{!}{\includegraphics{gauss.pdf}}
        \label{fig:gauss}
    }
    \subfloat[Pareto rewards.]{
        \resizebox{0.49\linewidth}{!}{\includegraphics{pareto.pdf}}
        \label{fig:pareto}
    }
    \caption{Cumulative regrets of the considered algorithms. We performed $20$ trials for each instance and reported mean $\pm$ std. The $6$ change-points are indicated by the vertical lines.}
    \label{fig:results_high}
\end{figure}
In this section, we evaluate how \algnameshort behaves in highly dynamic scenarios where change-points are close. \\ 
\textbf{Setting} We confront \algnameshort with two of the most popular algorithms from the literature, \texttt{Monitored UCB} \cite{cao2019nearly} and \texttt{Sliding Window UCB} \cite{garivier2011upper}. We consider two PS MABs: Gaussian rewards with $\sigma=1$ and Pareto rewards with $\epsilon < \frac{1}{2}$ and $v < 3$. In both MABs, we have $K=3$, $T=1500$ and $\Upsilon = 6$. Also, we have $\delta_{min}=0$, \emph{i.e.}, and some actions may not change their means after a change-point. However, at least one arm has its mean change, and the optimal action changes at least $4$ times. Interestingly, these instances violate Assumption \ref{ass:length}. We use $\sigma=1$ and the means reported in Table \ref{tab:gaussian_instance} for the Gaussian scenario. The optimal actions change $4$ times. For the Pareto scenario, we use $\epsilon=\frac{1}{2}$, $v=3$, and the means reported in Table \ref{tab:pareto_instance}. The optimal actions change $4$ times. In Figure \ref{fig:instance_high} we report the means of every action in every epoch.
\begin{figure}[htbp]
  \centering
  
  \begin{subtable}[t]{0.45\linewidth}   
    \footnotesize\begin{tabular}{
>{\columncolor[HTML]{FFFFFF}}c |
>{\columncolor[HTML]{FFFFFF}}c |
>{\columncolor[HTML]{FFFFFF}}c |
>{\columncolor[HTML]{FFFFFF}}c |
>{\columncolor[HTML]{FFFFFF}}c |
>{\columncolor[HTML]{FFFFFF}}c |
>{\columncolor[HTML]{FFFFFF}}c |
>{\columncolor[HTML]{FFFFFF}}l }
        & $E_1$                         & $E_2$                         & $E_3$                         & $E_4$                       & $E_5$                         & $E_6$                         & $E_7$                         \\ \hline
$\mu_1$ & \cellcolor[HTML]{FFFFC7}$1.2$ & $1.5$                         & $1.5$                         & \cellcolor[HTML]{FFFFC7}$2$ & \cellcolor[HTML]{FFFFC7}$1.8$ & $1.2$                         & $1.2$                         \\
$\mu_2$ & $1$                           & \cellcolor[HTML]{FFFFC7}$1.8$ & \cellcolor[HTML]{FFFFC7}$2.4$ & $1.8$                       & $1$                           & \cellcolor[HTML]{FFFFC7}$1.8$ & $1$                           \\
$\mu_3$ & $0.5$                         & $0.5$                         & $0.5$                         & $0.5$                       & $1$                           & $0.5$                         & \cellcolor[HTML]{FFFFC7}$1,7$
\end{tabular}
\caption{Gaussian instance.}
\label{tab:gaussian_instance}
  \end{subtable}\hfill
  \begin{subtable}[t]{0.48\linewidth}
   \footnotesize \begin{tabular}{
>{\columncolor[HTML]{FFFFFF}}c |
>{\columncolor[HTML]{FFFFFF}}c |
>{\columncolor[HTML]{FFFFFF}}c |
>{\columncolor[HTML]{FFFFFF}}c |
>{\columncolor[HTML]{FFFFFF}}c |
>{\columncolor[HTML]{FFFFFF}}c |
>{\columncolor[HTML]{FFFFFF}}c |
>{\columncolor[HTML]{FFFFFF}}l }
        & $E_1$                         & $E_2$                         & $E_3$                       & $E_4$                         & $E_5$                         & $E_6$                         & $E_7$                         \\ \hline
$\mu_1$ & \cellcolor[HTML]{FFFFC7}$1.2$ & $1.5$                         & \cellcolor[HTML]{FFFFC7}$2$ & $2$                           & $1.2$                         & $1.2$                         & $0.8$                         \\
$\mu_2$ & $1$                           & \cellcolor[HTML]{FFFFC7}$2.4$ & $1.8$                       & \cellcolor[HTML]{FFFFC7}$2.8$ & $1$                           & $1.5$                         & $2$                           \\
$\mu_3$ & $0.5$                         & $0.5$                         & $0.5$                       & $0.5$                         & \cellcolor[HTML]{FFFFC7}$1.7$ & \cellcolor[HTML]{FFFFC7}$1.7$ & \cellcolor[HTML]{FFFFC7}$2.9$
\end{tabular}
\caption{Pareto instance. }
\label{tab:pareto_instance}
  \end{subtable}

  \caption{Mean rewards per epoch. Cells highlighted in yellow contain the optimal reward for the corresponding epoch.}
  \label{fig:plot-plus-tables}
\end{figure}
\\
\textbf{Results} In Figure \ref{fig:results_high}, we report the cumulative regrets suffered by the considered algorithms. For each instance and algorithm, we performed $20$ trials and reported the average cumulative regrets with their aleatoric uncertainties. \algnameshort achieves, in both instances, a smaller cumulative regret than competitors. Moreover, it shows a smaller uncertainty and more stable performances across the trials, especially when rewards have infinite variance (Figure \ref{fig:pareto}). Interestingly, \algnameshort can outperform both \texttt{Monitored UCB} and \texttt{Sliding Window UCB} even when the rewards are Gaussian. This is probably because the change-points are frequent and very close. Robust mean estimation using median-of-means stabilizes the algorithm's behavior in data-scarce regimes. Finally, we remark that Assumption \ref{ass:length} is violated by these two instances; however, \algnameshort performs well (and so is \texttt{Monitored UCB}, which relies on a similar hypothesis). This phenomenon was already observed in \cite{cao2019nearly}, and shows how Assumption \ref{ass:length} is, in practice, is not very limiting.

\begin{figure}[h!]
    \centering
    \subfloat[Gaussian rewards.]{
        \resizebox{0.47\linewidth}{!}{\includegraphics{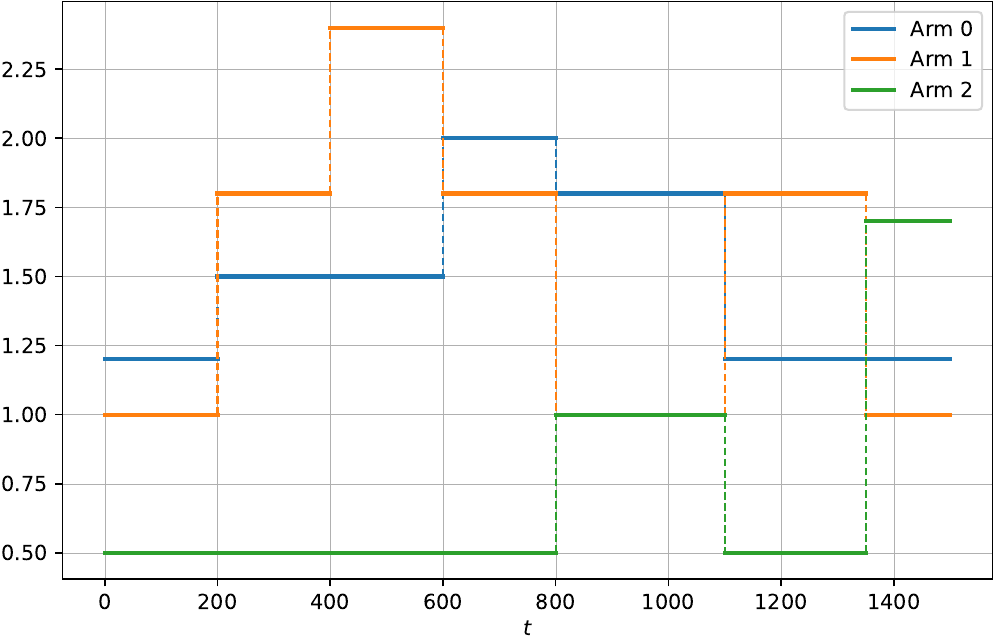}}
        \label{fig:gauss_instance}
    }
    \subfloat[Pareto rewards.]{
        \resizebox{0.47\linewidth}{!}{\includegraphics{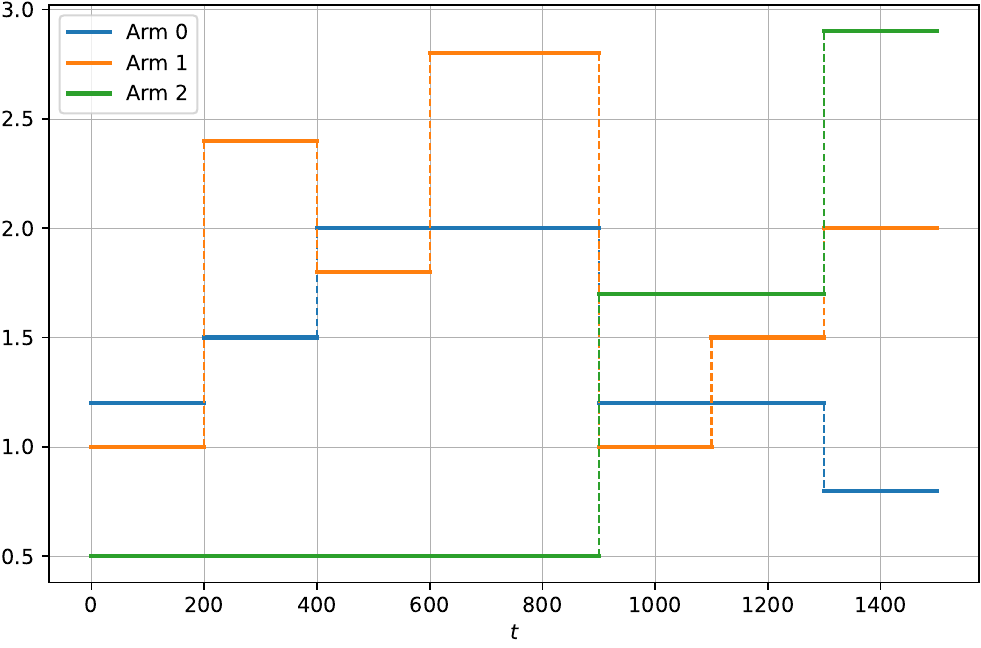}}
        \label{fig:pareto_instance}
    }
    \caption{Average rewards per epoch.}
    \label{fig:instance_high}
\end{figure}
\subsection{Sensibility to $\delta$}
In this section, we study the sensibility of \algnameshort to different magnitudes of changes. 
\paragraph{Setting} We consider four HTPS MABs with Pareto rewards ($\epsilon<1$, $v=1$), $K=3$, $T=500$, and $\Upsilon=1$. The starting means are $\mu_1 = 1$, $\mu_2=0.8$ and $\mu_2=0.5$, and a change occurs at $t_c=200$. We let $\delta_1 = 1$, $\delta_3 = 0$ (thus, $\delta_{min}=0$) and $\delta_2 \in \{1,2,5,10\}$, respectively. In the first PS MAB, the first action remains optimal from the start to the end of the trial; in the others, the second action becomes optimal after the change. In Figure \ref{fig:sensibility_instance}, we report the means of every action in every epoch for the four HTPS MABs.
\begin{figure}[t]
    \centering
    \resizebox{0.9\linewidth}{!}{\includegraphics{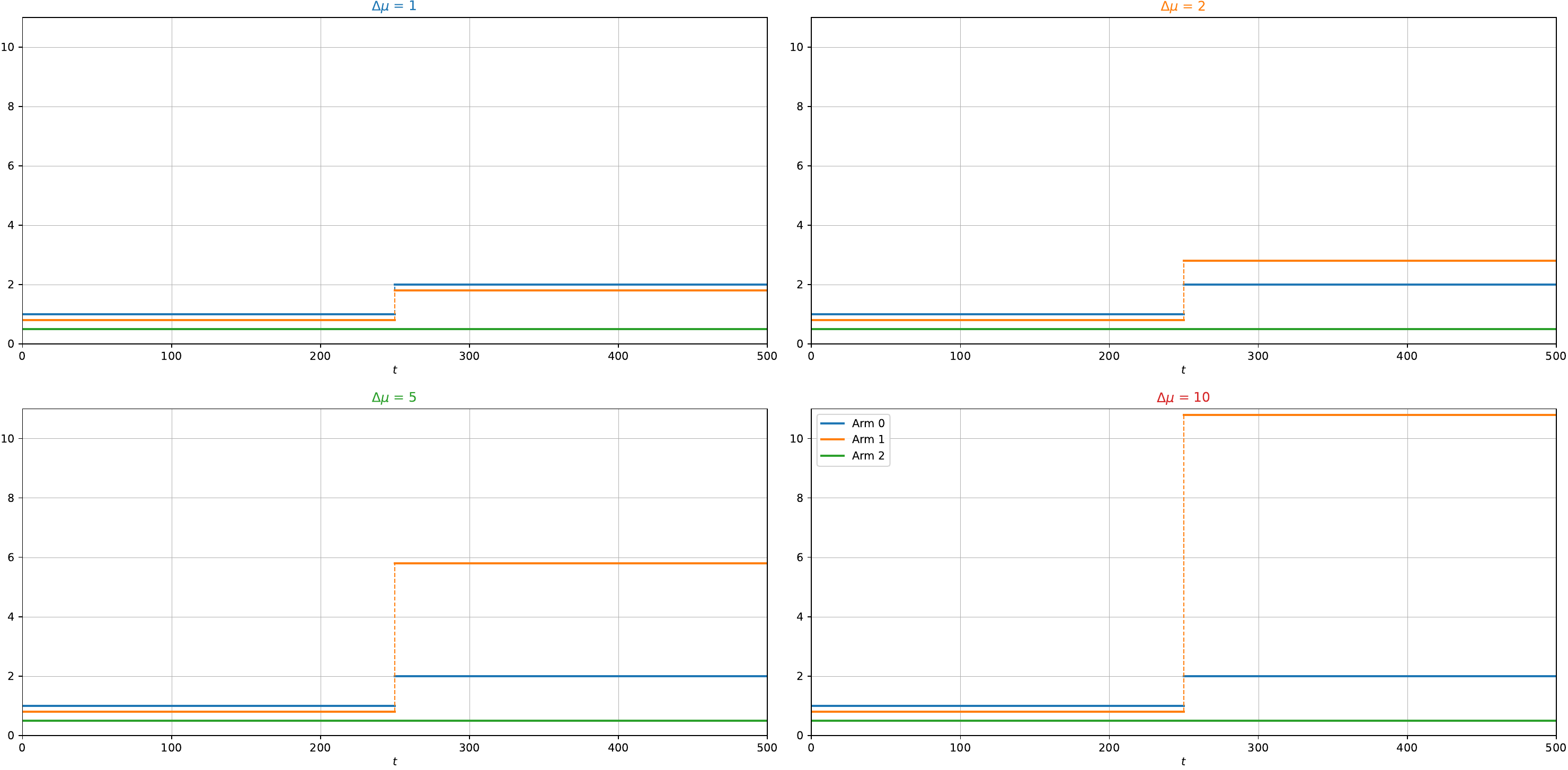}}  
    \caption{HTPS MABs with different magnitudes of change.  For all instances: $\Upsilon=1$, $K=3$, $\delta \in \{1,2,5,10\}$. Pareto noise with $\epsilon < 1$ and $v=3$.}
    \label{fig:sensibility_instance}
\end{figure}
\paragraph{Results} In Figure \ref{fig:sensibility_analysis}, we report the cumulative regrets suffered by \algnameshort in the four HTPS MABs. For each instance, we performed $20$ trials and reported the average cumulative regret (on the right), together with its standard deviation. Moreover, the dashed vertical lines indicate the average detection time and their standard deviations. As the four instances differ in terms of magnitude (\emph{e.g.}, in the first instance, the maximum mean is $2$, and in the fourth is $10.8$), we also reported the rescaled cumulative regrets (on the left). \algnameshort can detect the change with a reasonable delay, and all cumulative regrets show sublinear growth. As $\delta$ grows, the cumulative regret is larger, but the detection delay decreases. Intuitively, a larger change yields a larger regret but is also easier to detect. From the rescaled cumulative regrets, we can observe how a large $\delta$ w.r.t. to the mean does not deteriorate the performance of \algnameshort.
\begin{figure}[t]
    \centering
    \subfloat[Cumulative Regret.]{
        \resizebox{0.47\linewidth}{!}{\includegraphics{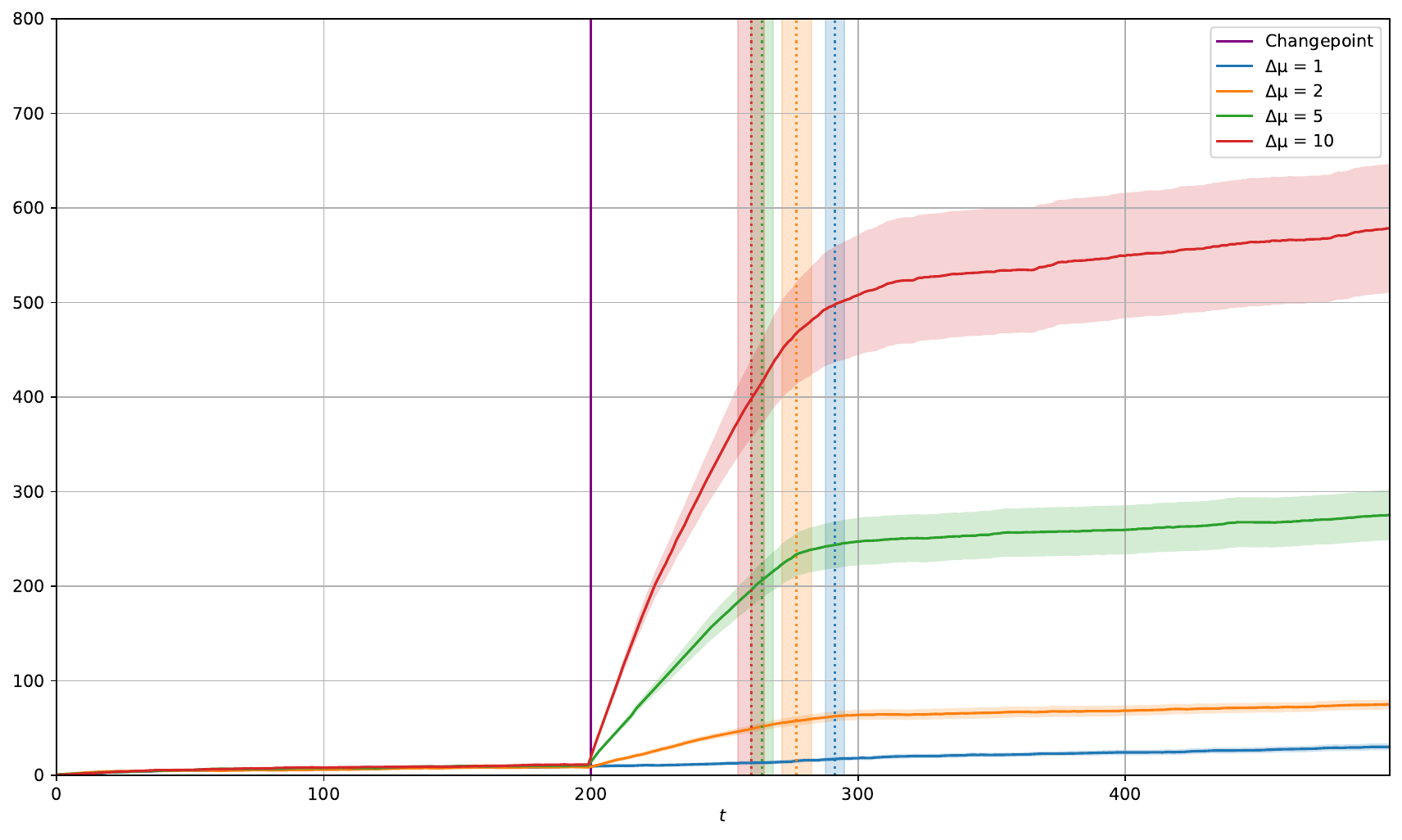}}
        \label{fig:delta}
    }
    \subfloat[Rescaled Cumulative Regret.]{
        \resizebox{0.47\linewidth}{!}{\includegraphics{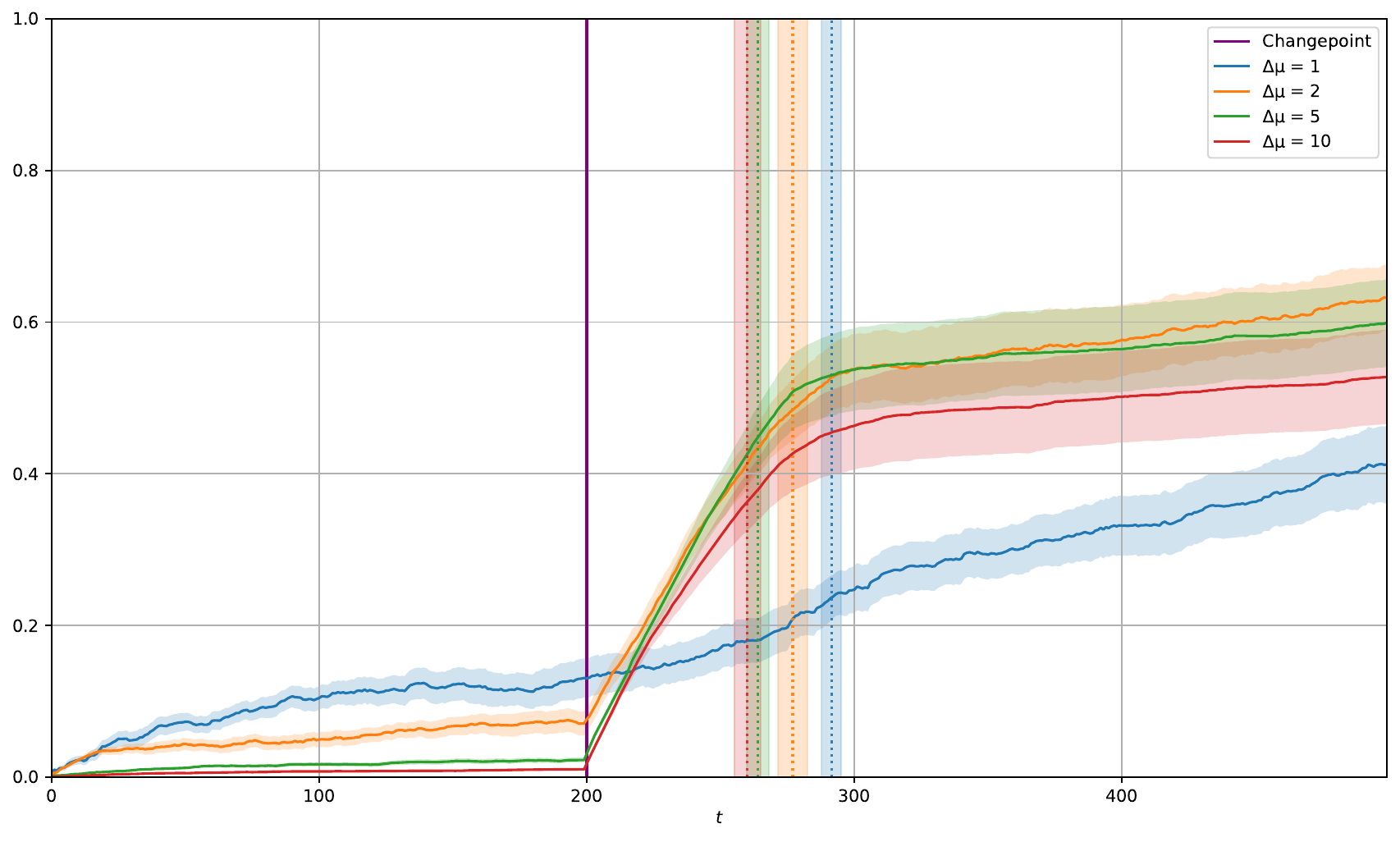}}
        \label{fig:delta_norm}
    }
    \caption{Cumulative regrets of \algnameshort in the four HTPS MABs represented in Figure 6. We performed $20$ trials for each instance and reported mean $\pm$ std. The purple vertical line indicates the change-point. The dashed vertical lines indicated the average detection time in the corresponding instance ($\pm$ std). On the left, we report the cumulative regrets of \algname. On the right, we have the same quantity rescaled by the maximum mean reward.}
    \label{fig:sensibility_analysis}
\end{figure}

\subsection{Stationary Environments}
In this section, we study the empirical behavior of \algnameshort in stationary HT MABs. 
\paragraph{Setting} We consider four HT MABs with Pareto rewards ($\epsilon<1$, $v=1$), $K=3$, $T=300$, and $\Upsilon=0$. We compare \algnameshort with two gold standards from the literature: \texttt{Robust UCB} \cite{bubeck2013bandits} and \texttt{MR-APE} \cite{lee2022minimax}. 

\begin{table}[t]
\centering
\begin{tabular}{
>{\columncolor[HTML]{FFFFFF}}c |
>{\columncolor[HTML]{FFFFFF}}c |
>{\columncolor[HTML]{FFFFFF}}c |
>{\columncolor[HTML]{FFFFFF}}l }
        & $\mu_1$  & $\mu_2$ &$\mu_3$                                        \\ \hline
Instance $1$ & $1$ & $0.5$  & $0.1$  \\
Instance $2$ & $1$ & $0.8$ &$0.7$  \\
Instance $3$ & $1$   & $0.9$   & $0.1$ \\
Instance $4$ & $1$   & $0.5$   & $0.5$ \\

\end{tabular}
\caption{Mean rewards per epoch in four stationary HT MABs. For all instances: $\Upsilon=0$, $K=3$. Pareto noise with $\epsilon < 1$ and $v=3$. }
\label{tab:stationary_instances}
\end{table}

\paragraph{Results} We remark that, in a stationary environment, the behavior of \algnameshort diverges from the one of \texttt{Robust UCB} in only two cases: (i) when there is a false detection (happens with probability smaller than $T^{-1}$) and (ii) \algnameshort performs a forced exploration (once every $\mathcal{O}(T^{-\frac{1}{2}})$ rounds). In both cases, the contribution to the regret is small compared to the dominant term, adding a constant factor at most. In Figure \ref{fig:stationary}, we report the cumulative regrets suffered by the algorithms in the four HT MABs. For each instance, we performed $20$ trials and reported the average cumulative regret, together with its standard deviation. \algnameshort raises only one false alarm in one trial of the fourth instance, and the average cumulative regret is thus slightly larger than the one \texttt{Robust UCB}, which is suited for the stationary setting. \algnameshort suffers cumulative regrets comparable to the ones of algorithms suited for the stationary case. All cumulative regrets show sublinear growth. 
\begin{figure}[t]
    \centering
    \resizebox{0.8\linewidth}{!}{\includegraphics{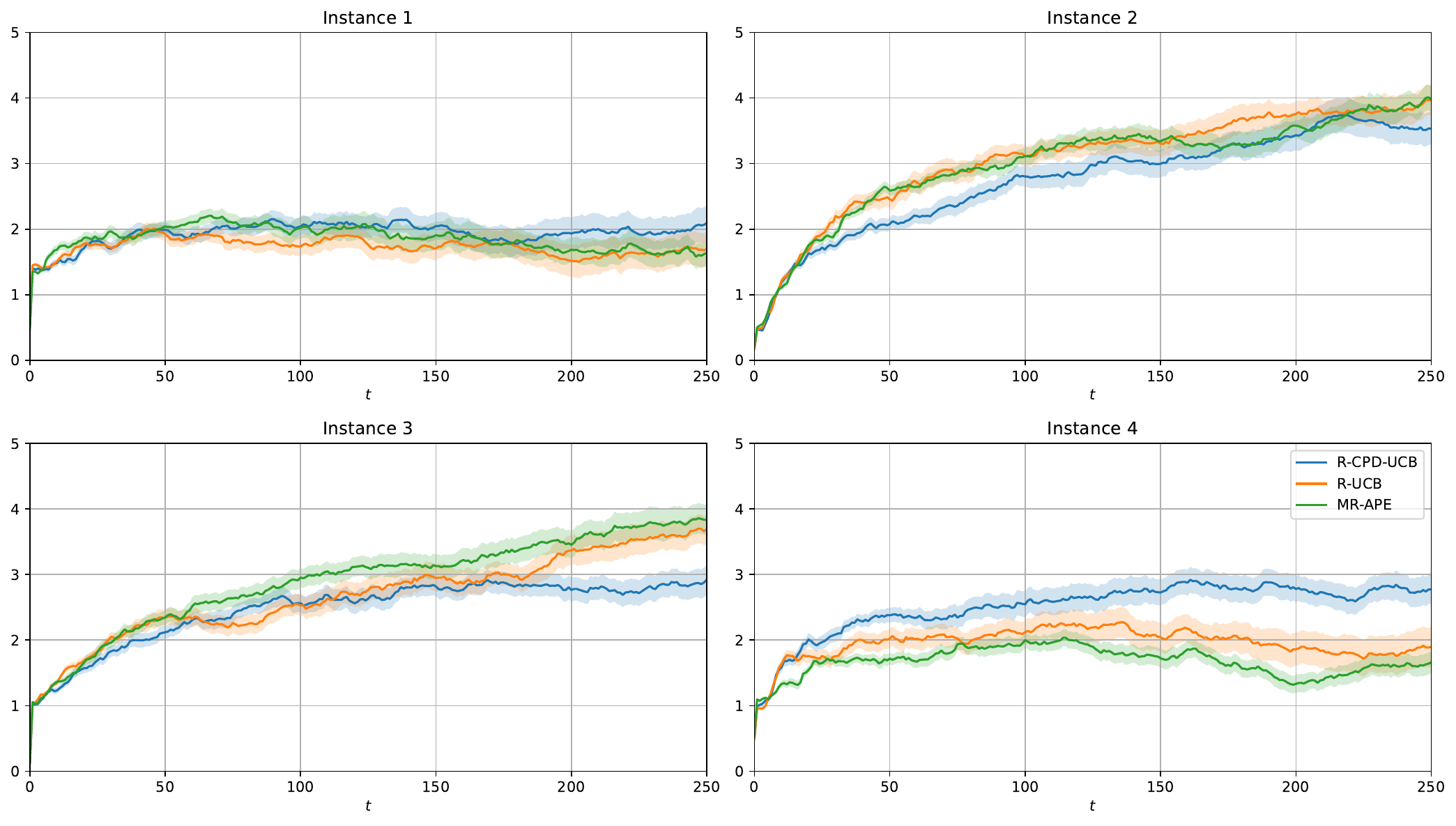}}
    \caption{Cumulative regrets of \algnameshort, \texttt{Robust UCB} \cite{bubeck2013bandits}, and \texttt{MR-APE} \cite{lee2022minimax} on the four HT MABs defined in Table 3. We performed $20$ trials for each instance and reported mean $\pm$ std.}
    \label{fig:stationary}
\end{figure}
\section{Computational Complexity of \algname}
In this section, we characterize the computational complexity of \algnameshort and provide a simple modification to the algorithm that allows for quicker computation while keeping similar theoretical guarantees. For the sake of traceability, we assume that all means belong to the set $[-M, M]$. We start by upper bounding the computational complexity of \algnameshort.
\begin{proposition}
    \label{prop:computational}
    Let $\widetilde{M} = \mathcal{O}\left( M + v^\frac{1}{1+\epsilon}\log(T)^\frac{\epsilon}{1+\epsilon}\right)$. Let $\xi$ be the machine tolerance. Then, \algnameshort takes at most $\mathcal{O}\left(KT^4 \log_2\left(\frac{\widetilde{M}}{\xi}\right)\right)$ steps with probability at least $1-\frac{1}{T}$.
\end{proposition}
\begin{proof}
    Using the bisection method with a tolerance of $\xi$, and searching in the interval $[-\widetilde{M}, \widetilde{M}]$, we can solve the two root-finding problems implied by Equation \eqref{eq:catoni_cs} in at most $\mathcal{O}\left(T\log_2\left(\frac{\widetilde{M}}{\xi}\right)\right)$. Note that, by Theorem 3.2 from \cite{bhatt2022catoni}, the solution lies in the search interval with probability at least $\Omega\left(1-\frac{1}{T}\right)$. Then, we observe that \algnameshort, at each round $t \in [T]$, for every action $i \in [K]$, runs a step of the \texttt{Catoni-FCS-detector} which computes $t$ CS of lengths $t, t-1, \ldots, 2, 1$. Computing a CS of length $t'\le T$ requires to compute $t'$ CIs, which requires $\mathcal{O}\left(t'\log_2\left(\frac{\widetilde{M}}{\xi}\right)\right)$ steps at most for each of them. Note that a solution always exists as the number of samples is always greater than $n_{min}$ when the CPD routine is executed. The result follows by upper bounding $t, t' \le T$.
\end{proof}
Proposition \ref{prop:computational} states that the computational complexity of \algnameshort is polynomial in $T$.
\end{document}